\documentclass[11pt,a4paper]{article}
\usepackage[margin=1in]{geometry}
\usepackage{microtype}
\usepackage{graphicx}
\usepackage{booktabs} 
\usepackage{tikz}
\usepackage{pgfplots}
\usetikzlibrary{positioning, shapes.geometric}
\usetikzlibrary{pgfplots.groupplots}
\pgfplotsset{compat=newest} 
\usepackage{color}
\makeatletter
\pgfplotsset{
    groupplot xlabel/.initial={},
    every groupplot x label/.style={
        at={($({\pgfplots@group@name\space c1r\pgfplots@group@rows.west}|-{\pgfplots@group@name\space c1r\pgfplots@group@rows.outer south})!0.5!({\pgfplots@group@name\space c\pgfplots@group@columns r\pgfplots@group@rows.east}|-{\pgfplots@group@name\space c\pgfplots@group@columns r\pgfplots@group@rows.outer south})$)},
        anchor=north,
    },
    groupplot ylabel/.initial={},
    every groupplot y label/.style={
            rotate=90,
        at={($({\pgfplots@group@name\space c1r1.north}-|{\pgfplots@group@name\space c1r1.outer
west})!0.5!({\pgfplots@group@name\space c1r\pgfplots@group@rows.south}-|{\pgfplots@group@name\space c1r\pgfplots@group@rows.outer west})$)},
        anchor=south
    },
    execute at end groupplot/.code={%
      \node [/pgfplots/every groupplot x label]
{\pgfkeysvalueof{/pgfplots/groupplot xlabel}};  
      \node [/pgfplots/every groupplot y label] 
{\pgfkeysvalueof{/pgfplots/groupplot ylabel}};  
    }
}

\def\endpgfplots@environment@groupplot{%
    \endpgfplots@environment@opt%
    \pgfkeys{/pgfplots/execute at end groupplot}%
    \endgroup%
}
\makeatother

\usepackage{hyperref}
\usepackage{algorithm}
\usepackage{algpseudocode}
\newcommand{\algoname}[1]{\textnormal{\textsc{#1}}}

\usepackage{amsmath}
\usepackage{amssymb}
\usepackage{mathtools}
\usepackage{amsthm}
\usepackage{bbm}
\usepackage{enumitem}
\usepackage[caption=false]{subfig}

\usepackage[capitalize,noabbrev]{cleveref}

\theoremstyle{plain}
\newtheorem{theorem}{Theorem}[section]
\newtheorem{proposition}[theorem]{Proposition}
\newtheorem{lemma}[theorem]{Lemma}
\newtheorem{corollary}[theorem]{Corollary}

\theoremstyle{definition}
\newtheorem{definition}[theorem]{Definition}

\theoremstyle{remark}
\newtheorem{remark}[theorem]{Remark}

\newcommand{\norm}[1]{\| #1 \|}
\newcommand{\abs}[1]{| #1 |}
\newcommand{\Norm}[1]{\left\| #1 \right\|}
\newcommand{\Abs}[1]{\left| #1 \right|}

\newcommand{\eps}{\epsilon}
\newcommand{\R}{\mathbb{R}}
\newcommand{\cO}{\mathcal{O}}
\newcommand{\inv}[1]{{\mathring{#1}}}
\newcommand{\NULL}{{\bot}}
\newcommand{\last}{\text{last}}

\newcommand{\opt}{\text{opt}}
\newcommand{\OL}{\textrm{OL}}

\DeclareMathOperator*{\E}{\mathbb{E}}
\DeclareMathOperator{\nnz}{nnz}
\DeclareMathOperator{\poly}{poly}

\DeclareMathOperator{\diag}{diag}

\title{Online Active Regression\footnote{A preliminary version appeared in the Proceedings of the 39th International Conference on Machine Learning (ICML 2022), PMLR 162, pp 3320--3335, 2022.}}
\author{Cheng Chen\qquad\quad Yi Li\qquad\quad Yiming Sun\\
School of Physical and Mathematical Sciences\\
Nanyang Technological University\\
\texttt{\{cheng.chen, yili\}@ntu.edu.sg, yiming005@e.ntu.edu.sg}
}
\date{}
\begin{document}
\maketitle

\begin{abstract}%
Active regression considers a linear regression problem where the learner receives a large number of data points but can only observe a small number of labels. Since online algorithms can deal with incremental training data and take advantage of low computational cost, we consider an online extension of the active regression problem: the learner receives data points one by one and immediately decides whether it should collect the corresponding labels. The goal is to efficiently maintain the regression of received data points with a small budget of label queries. We propose novel algorithms for this problem under $\ell_p$ loss where $p\in[1,2]$. To achieve a $(1+\epsilon)$-approximate solution, our proposed algorithms only require $\tilde{\mathcal{O}}(\epsilon^{-1} d \log(n\kappa))$ queries of labels, where $n$ is the number of data points and $\kappa$ is a quantity, called the condition number, of the data points. The numerical results verify our theoretical results and show that our methods have comparable performance with offline active regression algorithms.
\end{abstract}

\section{Introduction}

Linear regression is a simple method to model the relationship between the data points in a Euclidean space and their scalar labels. A typical formulation is to solve the minimization problem $\min_x \norm{Ax-b}_p$ for $A\in\R^{n\times d}$ and $b\in\R^n$, where each row $A_i$ is a data point in $\R^d$ and $b_i$ is its corresponding scalar label. When $p=2$, the linear regression is precisely the least-squares regression, which admits a closed-form solution and is thus a classical choice due to its computational simplicity. When $p\in [1,2)$, it is more robust than least-squares as the solution is less sensitive to outliers. A popular choice is $p=1$ because the regression can be cast as a linear programme though other values of $p$ are recommended depending on the distribution of the noise  in the labels. Interested readers may refer to Section 1.3 of~\cite{GM89} for some discussion.

One harder variant of linear regression is \emph{active regression}~\cite{SM14}, in which the data points are easy to obtain but the labels are costly. Here one can query the label of any chosen data point and the task is to minimize the number of queries while still being able to solve the linear regression problem approximately. Specifically, one constructs an index set $S\subset [n]$ as small as possible, queries $b_S$ (the restriction of $b$ on $S$) and computes a solution $\tilde{x}$ based on $A$, $S$ and $b_S$ such that
\begin{equation}\label{eqn:approx_guarantee}
\norm{A\tilde{x}-b}_p^p \leq (1 + \eps)\min_x \norm{Ax-b}_p^p.
\end{equation}
For $p=2$, the classical approach is to sample the rows of $A$ according to the leverage scores. This can achieve ~\eqref{eqn:approx_guarantee} with large constant probability using $|S| = O(d\log d + d/\eps)$ queries. \cite{CP19} reduced the query complexity to the optimal $O(d/\eps)$, based on graph sparsifiers. When $p=1$, \cite{CD21} and \cite{PPP21} showed that $O((d\log d)/\eps^2)$ queries suffice with large constant probability, based on sampling according to Lewis weights. More recently, \cite{MMWY} solved the problem for all values of $p$ with query complexity $\tilde{O}(d/\poly(\eps))$ for $p<2$ and $\tilde{O}(d^{p/2}/\poly(\eps))$ for $p>2$, where the dependence on $d$ is optimal up to logarithmic factors.

Another common setting of linear regression is the \emph{online setting}, which considers memory restrictions that prohibit storing the inputs $A$ and $b$ in their entirety. In such a case, each pair of data points and their labels (i.e. each row of $[A\ b]$) arrives one by one, and the goal is to use as little space as possible to solve the linear regression problem. Again, the case of $p=2$ has the richest research history, with the state-of-the-art results due to \cite{CMP} and \cite{JPW}, which retain only $O(\eps^{-1} d\log d\log(\eps\norm{A}_2^2))$ rows of $A$ (where $\norm{A}_2$ denotes the operator norm of $A$). The idea of the algorithms is to sample according to the online leverage scores, which was first employed by~\cite{Kapralov2017}. The online leverage score of a row is simply the leverage score of the row in the submatrix of $A$ consisting of all the revealed rows so far. The algorithm of \cite{JPW} is based on that of~\cite{CMP} with further optimized runtime. The case of $p = 1$ was solved by~\cite{sliding_window}, who generalized the notion of online leverage score to online Lewis weights and sampled the rows of $A$ according to the online Lewis weights.

In this paper, we consider the problem of \emph{online active regression}, a combination of the two variants above. In a similar vein to~\cite{CMP} and~\cite{JPW}, the rows of $A$ arrive one by one, and upon receiving a row, one must decide whether it should be kept or discarded and whether to query the corresponding label, without ever retracting these decisions. 
The problem was considered by \cite{RJZ17}, who assumed an underlying distribution of the data points together with a noise model of the labels and only considered $\ell_2$-regression. Here we do not make such assumptions and need to handle arbitrary input data. 
To the best of our knowledge, our work is the first to consider the online active regression in the general $\ell_p$-norm. 
Our approach is largely based upon the existing techniques for online regression and active regression. A technical contribution of our work is to show that one can compress a fraction of rows in a matrix by sampling these rows according to their Lewis weights while preserving the Lewis weights of the uncompressed rows (see Lemma~\ref{lem:compression-online-lewis-weights} for the precise statement), which may be of independent interest.

\paragraph{Our Results.} We show that the online active regression problem can be solved, attaining the error guarantee~\eqref{eqn:approx_guarantee} with constant probability, using $m=\cO (\eps^{-2}d \log (d/\eps) \log n \log \kappa^{\OL}(A))$ queries for $p=1$, $m = \cO( \eps^{-1} d \poly(\log (d/\eps)) \cdot \log(n \kappa^{\OL}))$ queries for $p\in (1,2)$ (where $\kappa^{\OL}$ is the online condition number of $A$) and $m = \cO(\eps^{-1}\allowbreak d \poly(\log (d/\eps))\log(n \norm{A}_2/\sigma))$ queries for $p = 2$ (where $\norm{A}_2$ is the operator norm of $A$ and $\sigma$ the smallest singular value of the first $d$ rows of $A$). Our algorithms are sublinear in space complexity, using $\cO(md)$ words.

The query complexity for $p\in (1,2)$ is only worse by a factor of $\log n \kappa^{\OL}$ than that of its offline counterpart~\cite{MMWY}, which is not unexpected, given that the same factor appears in the sketch size for the $\ell_1$-subspace embedding under the sliding window model~\cite{sliding_window}. 

We also demonstrate empirically the superior accuracy of our algorithm to online uniform sampling on both synthetic and real-world data. We vary the allotted number of queries and compare the relative error in the objective function of the regression (with respect to the minimum error, namely $\min_x \norm{Ax-b}_p$). For active $\ell_1$-regression, our algorithm achieves almost the same relative error as the offline active regression algorithm on both the synthetic and real-world data. For active $\ell_2$-regression, our algorithm significantly outperforms the online uniform sampling algorithm on both synthetic and real-world data and is comparable with the offline active regression algorithm on the synthetic data.
\section{Preliminaries}
\paragraph{Notation.} We use $[n]$ to denote the integer set $\{1,\dots,n\}$. For a matrix $A$, we denote by $A^\dagger$ its Moore–Penrose inverse.

For two matrices $A$ and $B$ of the same number of columns, we denote by $A\circ B$ the vertical concatenation of $A$ and $B$.

Suppose that $A\in \R^{n\times d}$. We define the operator norm of $A$, denoted by $\norm{A}_2$, to be $\max_{\norm{x}_2=1}\norm{Ax}_2$. We also define an \emph{online condition number} $\kappa^{\OL}(A) = \norm{A}_2\max_{S\subseteq [n],S\neq\emptyset} \norm{A_S^\dagger}_2$, where $A_S$ is the submatrix of $A$ consisting of the rows with indices in $S$.

Suppose that $A\in \R^{n\times d}$, $b\in \R^d$ and $p\geq 1$. We define $\Call{Reg}{A,b,p}$ to be an $x\in \R^d$ that minimizes $\norm{Ax-b}_p$. We remark that when $p>1$, the minimizer is unique.

\paragraph{Lewis weights.} A central technique to solve $\min_x \norm{Ax-b}_p$ is to solve a compressed version $\min_x \norm{SAx-Sb}_p$, where $S$ is a sampling matrix. This sampling is based on Lewis weights~\cite{CP15}, which are defined below.
\begin{definition}[Lewis weights]\label{def:lewis}
Suppose that $A\in\R^{n\times d}$ and $p\geq 1$. The $\ell_p$ Lewis weights of $A$, denoted by $w_1(A),\dots,w_n(A)$, are the unique real numbers such that 
\[
w_i(A) =(a_i^\top(A^\top W^{1-2/p}A)^{\dagger}a_i)^{p/2},
\] 
where $W$ is the diagonal matrix with diagonal elements $w_1(A),\dots,w_n(A)$ and $a_i\in \R^d$ is the $i$-th row of $A$ (viewed as a column vector).
\end{definition}

All Lewis weights $w_i(A)\in [0,1]$. For notational convenience, when $A$ has $n$ rows, we also write $w_n(A)$ as $w_{\last}(A)$.
The $\ell_2$ Lewis weight is also called the leverage score. 

A useful result of Lewis weights is their monotonicity when $p\leq 2$~\cite[Lemma 5.5]{CP15}, which we state formally below.
\begin{lemma}[Monotonicity of Lewis weights]\label{lem:monotonicity}
Suppose that $p\in [1,2]$, $A\in \R^{n\times d}$ and $r\in \R^d$. It holds for all $i=1,\dots,n$ that $w_i(A)\geq w_i(A\circ r^\top)$.
\end{lemma}

Next we revisit the classical result of $\ell_p$ subspace embedding using Lewis weights.
\begin{definition}
Given $p_1,\dots,p_n\in [0,1]$ and $p\geq 1$, the \emph{rescaled sampling matrix} $S$ with respect to $p_1,\dots,p_n$ is a random matrix formed by deleting all zero rows from a random $n\times n$ diagonal matrix $D$ in which $D_{i,i} = p_i^{-1/p}$ with probability $p_i$ and $D_{i,i} = 0$ with probability $1-p_i$.
\end{definition}

Associated with a rescaled sampling matrix $S$ are indicator variables $\{(\mathbbm{1}_S)_i\}_{i=1,\dots,n}$ (where $n$ is the number of columns of $S$) defined as follows. For each $i$, we define $(\mathbbm{1}_S)_i = 1$ if the $i$-th column of $S$ is nonzero, and $(\mathbbm{1}_S)_i = 0$ otherwise. 

\begin{lemma}[Lewis weight sampling \cite{CP15}]\label{lem:lewis_weights_sampling}
Let $A\in \R^{n\times d}$ and $p\geq 1$. Choose an oversampling parameter $\beta = \Theta(\log(d/\delta)/\eps^2)$ and sampling probabilities $p_1,\dots, p_n$ such that $\min\{\beta w_i(A),1\}\leq p_i\leq 1$ and let $S$ be the rescaled sampling matrix with respect to $p_1,\dots,p_n$. Then it holds with probability at least $1-\delta$ that $(1-\eps)\norm{Ax}_p\leq \norm{SAx}_p \leq (1+\eps)\norm{Ax}_p$ (i.e., $S$ is an  $\eps$-subspace embedding for $A$ in the $\ell_p$-norm) and $S$ has $O(\beta \sum_i w_i(A)) = O(\beta d)$ rows.
\end{lemma}

In the light of the preceding lemma, one can choose an $\eps$-subspace embedding matrix $S$ for $[A\ b]$ such that $S$ has only $\tilde{O}(d/\eps^2)$ rows and $\min_x\Norm{SAx-Sb}_p = (1\pm\eps)\min \Norm{Ax-b}_p$. The remaining question is how to compute the Lewis weights of a given matrix. \cite{CP15} showed that, for a given matrix $A\in\R^{n\times d}$, the following iterations
\begin{equation} \label{eq:lw_iter}
    W_{i,i}^{(j)}\leftarrow \left(a_i^\top\left(A^\top(W^{(j-1)})^{1-2/p}A\right)^\dagger a_i \right)^{p/2},
\end{equation}
with the initial point $W^{(0)}=I_n$, will converge to some diagonal matrix $W$, whose diagonal elements are exactly $w_1(A),\dots,w_n(A)$.

\begin{definition}[Online Lewis Weights]
Let $p\in[1,2)$ and $A\in\R^{n\times d}$. The online $\ell_p$ Lewis weights, denoted by $w_1^{\textrm{OL}}(A),\dots,w_n^{\textrm{OL}}(A)$, are defined to be $w_i^{\textrm{OL}}(A) = w_i(A^{(i)})$, where $A^{(i)}$ is the submatrix consisting of the first $i$ rows of $A$.
\end{definition}

By Lemma~\ref{lem:monotonicity}, we see that $w_i^{\OL}(A)\geq w_i(A)$. Hence, if we sample the rows of $A$ by the online Lewis weights, i.e. replacing $w_i(A)$ with $w_i^{\OL}(A)$ in the construction of $S$ in Lemma~\ref{lem:lewis_weights_sampling}, the resulting $S$ remains an $\eps$-subspace embedding for $A$ in the $\ell_p$-norm.

\paragraph{Preservation of $\ell_2$-Norm.} We shall need the Johnson-Lindenstrauss matrix for the online active $\ell_2$-regression.
\begin{definition}[Johnson-Lindenstrauss Matrix] \label{def:JL matrix}
Let $X\subseteq \R^d$ be a point set. A matrix $J$ is said to be a Johnson-Lindenstrauss matrix for $X$ of distortion parameter $\eps$ (or, an $\eps$-JL matrix for $X$) if $(1-\eps)\norm{x}_2^2 \leq \Norm{Jx}_2^2 \leq (1+\eps) \norm{x}_2^2$ for all $x\in X$.
\end{definition}
It is a classical result~\cite{KN14} that when $|X| = T$, there exists a random matrix $J \in \R^{m \times d}$ with $m = \cO(\eps^{-2}\log (T/\delta))$ such that (i) $J$ is an $\eps$-JL matrix for $X$ with probability at least $1-\delta$, (ii) each column of $J$ contains $\cO(\eps^{-1}\log(T/\delta))$ nonzero entries and (iii) $J$ can be generated using $\cO(\log^2 (|T|/\delta)\log d)$ bits.

\section{Algorithms and Main Results}\label{sec:main_results}
In this section, we shall demonstrate how to solve the online active regression with $\tilde{\cO}(\eps^{-2}d\log(n\kappa^{\OL}(A)))$ queries. Then we shall improve the number of queries to $\tilde{\cO}(\eps^{-1}d\log(n\kappa^{\OL}(A)))$ in Section~\ref{sec:optimal_eps}, attaining the optimal dependence on $\eps$.

The high-level approach follows~\cite{MMWY} and we give a brief review below. We sample $A$ twice w.r.t. (online) Lewis weights but with different oversampling parameters $\beta$, getting $SA$ of $\cO(d \log d)$ rows and $S_1A$ of $\cO(d^2\poly(\eps^{-1}\log d))$ rows, respectively. We use the sketching matrix $S$ to solve $\min_{x\in\R^d}\norm{Ax-b}_p$, obtaining a constant-factor approximation solution $x_c$. The problem is then reduced to solving $\min_{x\in\R^d} \norm{Ax-z}_p$ with $z = b-Ax_c$, for which we shall solve $\min_{x\in \R^d}\norm{S_1Ax-S_1z}_p$ instead. Since $S_1A$ has $\Omega(d^2)$ rows, we repeat the idea above and further subsample $S_1A$ twice with different sampling parameters, getting $S_2S_1A$ of $\cO(d\log d)$ rows and $S_3S_1A$ of $\cO(d\poly(\eps^{-1}\log d))$ rows. The sampling matrix $S_2$ is used to obtain a constant-factor approximation solution $\hat{x}_c$ to $\min_{x\in\R^d}\norm{S_1Ax-S_1z}_p$ and $S_3$ is used to solve $\min_{x\in\R^d}\norm{S_1Ax-(S_1z-S_1Ax'_c)}_p$ with a near-optimal solution $\bar{x}'$. The near-optimal solution to $\min_{x\in\R^d} \norm{S_1Ax-S_1z}_p$ is then $\bar{x} = \hat{x}_c+\bar{x}'$. Finally, the  solution to the original problem is $\tilde{x} = x_c+\bar{x}$.

We shall maintain in parallel four independent (rescaled) row-sampled submatrices of the input matrix $A$, denoted by $\tilde{A}=SA$, $\tilde{A}_1=S_1A$, $\tilde{A}_2=S_2\tilde{A}_1$ and $\tilde{A}_3=S_3\tilde{A}_1$, where $S,S_1,S_2,S_3$ are rescaled sampling matrices. Recall that $z = b - Ax_c$, where $x_c$ is a constant factor approximation. The corresponding sampled subvectors of $b$ and $z$ are $\tilde{b} = Sb$,
$\tilde{b}_2=S_2S_1b$, $\tilde{b}_3=S_3S_1b$,
$\tilde{z}_2 = \tilde{b}_2-\tilde{A}_2x_c$ and $\tilde{z}_3 = \tilde{b}_3-\tilde{A}_3x_c$, respectively. We shall keep updating these sampled submatrices and vectors. We make the rows of $A$ global variables and $a_t\in \R^{d}$ is the $t$-th row of $A$. We denote by $A^{(t)}$ the first $t$ rows of $A$ and $b^{(t)}$ the first $t$ coordinates of $b$. Furthermore, in our presentation of the algorithm, for a variable $X$, we denote by $X^{(t)}$ its value at the $t$-th stage in the online algorithm.


\begin{algorithm}[!htp]
\caption{Online Active Regression for $p \in (1,2)$}
\label{alg:p-simple-OAR}
\textbf{Initialize:} Let $\tilde{A}^{(d)}, \tilde{A}_1^{(d)}, \tilde{A}_2^{(d)}, \tilde{A}_3^{(d)}$ be the first $d$ rows of $A$ and $\tilde{b}^{(d)}$ be the first $d$ rows of $b$. 
\begin{algorithmic}[1]
\State $\beta \gets \Theta(\log d)$
\State $\beta_1 \gets \Theta(d \log(1/\eps\delta) / \eps^{2+p} )$
\State $\beta_{2} \gets \Theta(\log d)$
\State $\beta_{3} \gets \Theta(\log^2 d \log(d/\eps) \log(1/\delta) / \eps^{2})$
\State  Retain the first $d$ rows of $A$ \label{alg:line:start_reading_stream}
\While{there is an additional row $a_{t}$}
    \State $\tilde{w}_{t} \gets w_{t}(A^{(t)})$     \label{alg:line:online_lewis}
    \State $p_t \gets \min\{\beta\tilde{w}^{(t)}, 1\}$
    \State $(\tilde{A}^{(t)}, \tilde{b}^{(t)}) \gets \Call{Sample}{a_{t}, p_{t}, \tilde{A}^{(t-1)}, \tilde{b}^{(t-1)}, p}$
    \State $\tilde{w}_{1,t} \gets w_{t}(A^{(t)})$    \label{alg:line:online_lewis1} 
    \State $p_{1,t} \gets \min\{\beta_1 \tilde{w}_{1,t}, 1\}$
    \State Sample $a_t$ with probability $p_{1,t}$
    \If{$a_t$ is sampled}
        \State $\tilde{A}_{1}^{(t)} \gets \tilde{A}_{1}^{(t-1)} \circ a_t^{\top} p_{1,t}^{-1/p}$
        \State $\tilde{w}_{2,t} \gets w_{\last}(\tilde{A}_1^{(t)})$ \label{alg:line:online_lewis2}
        \State $p_{2,t} \gets \min \{ \beta_2 \tilde{w}_{2,t}, 1 \}$
        \State $(\tilde{A}_{2}^{(t)}, \tilde{b}_{2}^{(t)}) \gets \Call{Sample}{a_t p_{1,t}^{-1/p}, p_{2,t}, \tilde{A}_{2}^{(t-1)}, \tilde{b}_{2}^{(t-1)}, p}$
        \State $\tilde{w}_{3,t} \gets w_{\last}(\tilde{A}_1^{(t)})$ \label{alg:line:online_lewis3}
        \State $p_{3,t} \gets \min \{ \beta_3 \tilde{w}_{3,t}, 1 \}$
        \State $(\tilde{A}_{3}^{(t)}, \tilde{b}_{3}^{(t)}) \gets \Call{Sample}{a_t p_{1,t}^{-1/p}, p_{3,t}, \tilde{A}_{3}^{(t-1)}, \tilde{b}_{3}^{(t-1)}, p}$
    \EndIf
\EndWhile \label{alg:line:finish_reading_stream}
\State $x_c \gets \Call{Reg}{\tilde{A}, \tilde{b}, p}$
\State $\tilde{z}_{2} \gets \tilde{b}_{2} - \tilde{A}_{2} x_c$
\State $\hat{x}_c \gets \Call{Reg}{\tilde{A}_{2}, \tilde{z}_{2}, p}$
\State $\tilde{z}_{3} \gets \tilde{b}_{3} - \tilde{A}_{3} x_c$
\State $\bar{x}' \gets \Call{Reg}{\tilde{A}_{3}, \tilde{z}_{3} - \tilde{A}_{3} \hat{x}_{c}, p}$
\State $\bar{x} \gets \hat{x}_c + \bar{x}'$
\State $\tilde{x} \gets x_{c} + \bar{x}$
\State \Return $\tilde{x}$
\end{algorithmic}
\end{algorithm}

\begin{algorithm}[!htp]
\caption{\algoname{Sample}($a_{t}, p_{t}, \tilde{A}^{(t-1)}, \tilde{b}^{(t-1)}, p$)}
\label{alg:p-simple-sample}
\begin{algorithmic}[1]
\State Sample $a_t$ with probability $p_t$
    \If{$a_t$ is sampled}
        \State \textbf{Query} $b_t$
        \State $(\tilde{A}^{(t)}, \tilde{b}^{(t)}) \gets (\tilde{A}^{(t-1)} \circ a_t^{\top} p_{t}^{-1/p}, \tilde{b}^{(t-1)} \circ b_t p_t^{-1/p})$
    \Else 
        \State $(\tilde{A}^{(t)}, \tilde{b}^{(t)}) \gets (\tilde{A}^{(t-1)}, \tilde{b}^{(t-1)})$
    \EndIf

\end{algorithmic}
\end{algorithm}

\subsection{The case \texorpdfstring{$p\in (1,2]$}{Lg}} \label{sec:p_between_1_2}

We present our main algorithm for $p\in (1,2]$ in Algorithm~\ref{alg:p-simple-OAR}. The following is the guarantee of the algorithm.

\begin{theorem}\label{thm:p-naive-OAR}
Let $A \in \R^{n\times d}$ and $b \in \R^{n}$. Algorithm~\ref{alg:p-simple-OAR} outputs a solution $\tilde{x}$ which satisfies that
\begin{equation}\label{eqn:ell_p_guarantee}
\Norm{A\tilde{x} - b}_p \leq (1+\eps) \min_{x\in\R^d}  \Norm{Ax-b}_p
\end{equation}
with probability at least $0.98-\delta$ and makes
\[
\cO \left(\frac{d}{\eps^2}\log^2 d \log^2 \frac{d}{\eps} \cdot \log \frac{n\kappa^{\OL}(A)}{\delta} \log \frac{1}{\delta} \right)
\]
queries overall in total.
\end{theorem}

A major drawback of Algorithm~\ref{alg:p-simple-OAR} is the cost of calculating the online Lewis weights.  Recall that the online Lewis weight of $a_t$ is defined with respect to the first $t$ rows of $A$. A na\"ive implementation would require storing the entire matrix $A$, partly defying the purpose of an online algorithm. Furthermore, the iterative procedure described after Definition~\ref{def:lewis} takes $\cO(\log t)$ iterations to reach a constant-factor approximation to the Lewis weights~\cite{CP15}, where each iteration takes $\cO( td^2 + d^3 )$ time, which would become intolerable as $t$ becomes large. To address this issue, we adopt the compression idea in~\cite{sliding_window}, which maintains $O(\log n)$ rescaled row-sampled submatrices of $A$, each having a small number of rows. The `compression' algorithm is presented in Algorithm~\ref{alg:compression}.

\begin{algorithm}[!htp]
\caption{Compression algorithm for calculation of online Lewis weights}\label{alg:compression}
\textbf{Initialize:} $B_0$ contains the first $d$ rows of $A$; $B_1,\dots,B_{\log n}$ are empty matrices; $Q = \Theta(\eps^{-2}d\log^3 n)$.
\begin{algorithmic}[1]
    \State $\beta \gets \Theta(\eps^{-2} \log(n/\delta) \log^2 n)$
    \While{there is an additional row $a_t$}
        \State $B_0\gets B_0 \circ a_t$
        \If{ the size of $B_0$ exceeds $Q$ }
            \State $j \gets $ the smallest index $i$ such that $B_i$ is empty
            \State $M\gets B_{i-1}\circ B_{i-2}\circ \cdots \circ B_0$
            \State $p_i\gets \min\{\beta w_i(M),1\}$ for all $i$
            \State $S\gets $ rescaled sampling matrix with respect to probabilities $\{p_i\}_i$
            \State $B_i \gets SM$
            \State $B_0,B_1,\dots,B_{i-1}\gets$ empty matrix
        \EndIf
    \EndWhile
\end{algorithmic}
\end{algorithm}

With the compression algorithm for $A$ which maintains $B_0,\dots,B_{\log n}$, we can replace Line~\ref{alg:line:online_lewis} of Algorithm~\ref{alg:p-simple-OAR} with
\begin{equation}\label{eqn:compressed_update}
\tilde{w}_t \gets w_{\last}(B_{\log n}\circ B_{\log n - 1}\circ\cdots\circ B_0).
\end{equation}
Similarly, we run an additional compression algorithm for each of $\tilde{A}_1$ and replace Lines~\ref{alg:line:online_lewis1}, \ref{alg:line:online_lewis2} and \ref{alg:line:online_lewis3} with updates analogous to~\eqref{eqn:compressed_update}.

By the construction of the blocks $B_i$'s, each $B_i$ contains at most $R = \cO(\eta^{-2} d \log (n/\delta) \log^2 n )$ rows with probability at least $1 - \delta/\poly(n)$, sufficient for taking a union bound over all the blocks throughout the process of reading all $n$ rows of $A$. Hence we may assume that each block $B_i$ always contains at most $R$ rows. Now, $\tilde{w}_t$ is calculated to be the Lewis weight of a matrix of $R' = \cO(Q + R\log n) = \cO(R\log n)$ rows, which can be done in $\cO((R' d^2 + d^3)\log R') = \cO(\eta^{-2}d^3\poly(\log(n/\delta \eta)))$ time for a $(1\pm\eta)$-factor approximation of Lewis weights, where the dependence on $n$ is only polylogarithmic. The remaining question is correctness and the following theorem is the key to proving the correctness.

\begin{theorem}
\label{thm:compression-matrix}
Let $A\in \R^{n\times d}$. With Algorithm~\ref{alg:compression} maintaining $B_0,\dots,B_{\log n}$, let $\tilde{w}_{t}$ be as in~\eqref{eqn:compressed_update} for each $t\leq n$. Then it holds with probability at least $1-\delta/\poly(n)$ that
\[
(1 - \eta)w_{t}(A^{(t)}) \leq  \tilde{w}_{t} \leq (1 + \eta)w_{t}(A^{(t)}), \quad \forall t\leq n,
\] 
where $A^{(t)}$ is the submatrix consisting of the first $t$ rows of $A$. The weights $\tilde{w}_{t}$ can be calculated in $\cO(\frac{d^3}{\eta^2} \poly(\log\frac{n}{\delta\eta}))$ time and Algorithm~\ref{alg:compression} needs $\cO(\frac{d^2}{\eta^{2}} \poly(\log \frac{n}{\delta}))$ words of space overall in total.
\end{theorem}

The proof of Theorem~\ref{thm:compression-matrix} is deferred to Section~\ref{sec:approx_online_lewis_weights_proof}. Taking $\eta$ to be a constant in Theorem~\ref{thm:compression-matrix} for constant-factor approximations to Lewis weights, we can now strengthen Theorem~\ref{thm:p-naive-OAR} as follows.

\begin{theorem}[Strengthening Theorem~\ref{thm:p-naive-OAR}]
\label{thm:p-compression-OAR}
Let $A\in \R^{n\times d}$ and $b\in \R^n$. Algorithm~\ref{alg:p-simple-OAR} outputs a solution $\tilde{x}$ which satisfies~\eqref{eqn:ell_p_guarantee} with probability at least $0.98-\delta$, making 
\[
m = \cO\left(\frac{d}{\eps^2} \log^2 d \log^2 \frac{d}{\eps} \cdot \log \frac{n\kappa^{\OL}(A)}{\delta} \log \frac{1}{\delta} \right)
\]
queries. Furthermore, when implemented using the compression technique as explained above, with probability at least $0.98 - \delta$, Algorithm~\ref{alg:p-simple-OAR} uses $\cO(md)$ words of space in total and uses $\cO(n d^3 \poly(\log(n/\delta)))$ time to process the data stream (Lines \ref{alg:line:start_reading_stream}--\ref{alg:line:finish_reading_stream}). 
\end{theorem}

The failure probability in Theorem~\ref{thm:p-compression-OAR} is $0.02+\delta$. One $0.01$ comes from obtaining constant-factor approximations $x_c$ and $\hat x_c$. This $0.01$ failure probability can be reduced to $\delta$ by employing the same boosting procedure in \cite{MMWY} which computes $\log(1/\delta)$ solutions, each being a constant-factor approximation with a constant probability, and then finds a good solution among them. The other $0.01$ comes from bounding $\norm{S_1 z}_p^p = \cO(\norm{z}_p^p)$ and $\norm{S_3 S_1 z}_p^p = \cO(\norm{S_1 z}_p^p)$. The failure probability of bounding $\norm{S_3 S_1 z}_p^p$ can be reduced to $\delta$ by employing the same boosting procedure in \cite[Section 4.2.2]{MMWYv1}, which uses $\log\frac{1}{\delta}$ independent copies of $S_3$, removes the largest $10\%$ of $\norm{S_3 S_1 z}_p^p$ and chooses an arbitrary remaining $\norm{S_3 S_1 z}_p^p$. 
For $\norm{S_1 z}_p^p$, we use Markov's inequality, obtaining that $\norm{S_1 z}_p^p \leq \frac{\norm{z}^p}{\delta}$ with probability at least $1 - \delta$. Rescaling $\eps = \eps \delta$ yields $\beta_1 = \frac{d}{\delta^{2+p} \eps^{2+p}} \log \frac{1}{\eps \delta}$. Hence, the algorithm's overall failure probability can be reduced to $\delta$ (after rescaling) while maintaining asymptotically the same query and space complexity.

To conclude, the theoretical guarantee of Algorithm~\ref{alg:p-simple-OAR}, with the aforesaid modification for boosting the success probability, is as follows.

\begin{theorem}[$1-\delta$ success probability]
\label{thm:p-compression-OAR-boost}
Let $A\in \R^{n\times d}$ and $b\in \R^n$. There exists an algorithm outputting a solution $\tilde{x}$ which satisfies~\eqref{eqn:ell_p_guarantee} with probability at least $1-\delta$, making 
\[
m = \cO\left(\frac{d}{\eps^2} \log^2 d \log^2 \frac{d}{\eps} \cdot \log \frac{n\kappa^{\OL}(A)}{\delta} \log^2 \frac{1}{\delta} \right)
\]
queries. Furthermore, with probability at least $1-\delta$, the algorithm uses $\cO(md)$ words of space in total and uses $\cO(n d^3 \poly(\log(n/\delta)))$ time to process the data stream. 
\end{theorem}
\subsection{The case \texorpdfstring{$p=2$}{p=2}}

\begin{algorithm}[t]
\caption{Online Active Regression for $p=2$}
\label{alg:online-active-regression}
\textbf{Initialize:} Let $\tilde{A}^{(d)},\tilde{A}_1^{(d)},\tilde{A}_2^{(d)},\tilde{A}_3^{(d)}$ be the first $d$ rows of $A$ and $\tilde{b}^{(d)},\tilde{b}_2^{(d)},\tilde{b}_3^{(d)}$ be the first $d$ rows of $b$.  Let $x_{c}^{(d)} = \Call{Reg}{\tilde{A}^{(d)}, \tilde{b}^{(d)}, 2}$, $\tilde{z}_{2}^{(d)}=\tilde{z}_{3}^{(d)}=\tilde{b}^{(d)}-\tilde{A}^{(d)}x_{c}^{(d)}$, $\hat{x}_{c}^{(d)} = \Call{Reg}{\tilde{A}_{2}^{(d)}, \tilde{z}_{2}^{(d)}, 2}$ and $\bar{x}'_{d} = \Call{Reg}{\tilde{A}_{3}^{(d)}, \tilde{z}_{3}^{(d)}-\tilde{A}_{3}^{(d)}\hat{x}_{c}^{(d)}, 2}$. 
Let $\inv{G}^{(d)} = ((\tilde{A}^{(d)})^\top \tilde{A}^{(d)})^{-1}$ and $H^{(d)}=\tilde{A}^{(d)}\inv{G}^{(d)}$. Also let $\inv{G}_i^{(d)} = ((\tilde{A}_{i}^{(d)})^\top \tilde{A}_{i}^{(d)})^{-1}$ and $H_{i}^{(d)}=\tilde{A}_i^{(d)}\inv{G}_i^{(d)}$ for $i=1,2,3$. Let $J^{(d)} \in \R^{\cO(\log\frac{n}{\delta}) \times d}$ be a constant-factor approximation JL matrix.
\begin{algorithmic}[1]
\State $\beta \gets \Theta(\log d)$
\State $\beta_1 \gets \Theta((d\log(1/\eps) + \log(1/\delta)) / \eps^4)$
\State $\beta_2 \gets \Theta(\log d)$
\State $\beta_3 \gets \Theta((\log^2 d) \log(d/\eps) \log(1/\delta)/\eps^2)$
\State retain the first $d$ rows of $A$ 
\While{there is an additional row $a_t$}
	\State $\tilde{w}_t \gets \norm{H^{(t-1)}a_t}_2^2$
	
	\State $(x_{c}^{(t)}, \tilde{A}^{(t)}, \tilde{b}^{(t)}, \inv{G}^{(t)}, H^{(t)} )\gets
	\Call{SampleQuery}{a_t, \tilde{b}^{(t-1)}, \NULL, \NULL, 
	\tilde{A}^{(t-1)},
         \beta, \tilde{w}_t, \inv{G}^{(t-1)},\!1}$

	\State $\tilde{w}_{1,t} \gets \norm{H_1^{(t)}a_t}_2^2$
	\State $p_{1,t} \gets \min\{\beta_1 \tilde{w}_{1,t},1\}$
	\State Sample $a_t$ with probability $p_{1,t}$
	\If{$a_t$ is sampled}
	    \State $\tilde{A}_{1}^{(t)} \gets \tilde{A}_{1}^{(t-1)}\circ \frac{a_t^\top}{\sqrt{p_{1, t}}}$
	    
	    \State $(\inv{G}_{1}^{(t)}, H_{1}^{(t)})  \gets \Call{Update}{\frac{a_t}{\sqrt{p_{1,t}}}, \NULL, \NULL,
	    \tilde{A}_{1}^{(t - 1)},  \inv{G}_{1}^{(t - 1)}}$
        
        \State $\tilde{w}_{2,t} \gets \norm{H_2^{(t)}\frac{a_t}{\sqrt{p_{1,t}}}}_2^2$
 
        \State $(\hat{x}_{c}^{(t)}\!, \tilde{A}_{2}^{(t)}\!, \tilde{b}_{2}^{(t)}\!, \inv{G}_{2}^{(t)}\!, H_{2}^{(t)}) \gets \Call{SampleQuery}{\frac{a_t}{\sqrt{p_{1,t}}}, \tilde{b}_{2}^{(t-1)}, x_c^{(t)}\!, \NULL, \tilde{A}_{2}^{(t-1)}\!,
            \beta_2, \tilde{w}_{2,t}, \inv{G}_{2}^{(t-1)}\!, 2}$
 
        \State $\tilde{w}_{3,t}=\norm{H_3^{(t)}\frac{a_t}{\sqrt{p_{1,t}}}}_2^2$
        \State $(\bar{x}'^{(t)}\!, \tilde{A}_{3}^{(t)}\!, \tilde{b}_{3}^{(t)}\!, \inv{G}_{3}^{(t)}\!, H_{3}^{(t)}) \gets
            \Call{SampleQuery}{\frac{a_t}{\sqrt{p_{1,t}}}, \tilde{b}_{3}^{(t-1)}\!, x_c^{(t)}\!,\hat{x}_{c}^{(t)}\!, \tilde{A}_{3}^{(t - 1)}\!, 
           \beta_3, \tilde{w}_{3,t}, \inv{G}_{3}^{(t-1)}\!, 3}$
    \EndIf
    \State$\bar{x}^{(t)}\gets \hat{x}_{c}^{(t)}+\bar{x}'^{(t)}$
	\State $\tilde{x}^{(t)}\gets \bar{x}^{(t)}+x_{c}^{(t)}$
\EndWhile
\State \Return $\tilde{x}^{(t)}$
\end{algorithmic}
\end{algorithm}

\begin{algorithm}[t]
\caption{$\algoname{SampleQuery}(a_t, \tilde{b}^{(t-1)}, x_{c}^{(t)}, \hat{x}_c^{(t)}, \tilde{A}^{(t-1)}, \beta, \tilde{w}_t, \inv{G}^{(t-1)},\chi)$ in Algorithm~\ref{alg:online-active-regression}}
\label{alg:online-sample}

\begin{algorithmic}[1]

  \State $p_{t} \gets \min\{\beta \tilde{w}_t,1\}$
  \State Sample $a_t$ with probability $p_{t}$
  \If{$a_t$ is sampled}
    \State $\tilde{A}^{(t)}\gets \tilde{A}^{(t-1)}\circ \frac{a_t^{\top}}{\sqrt{p_{t}}}$
    \State \textbf{Query} $b_t$
    \If{$\chi = 1$}
         \State $\tilde{b}^{(t)} \gets \tilde{b}^{(t-1)}\circ \frac{b_t}{\sqrt{p_{t}}}$
    \Else
        \State $b^{(t)} \gets b^{(t-1)} \circ \frac{b_t}{\sqrt{p_{1\!,\!t} p_{t}}}$
        \State $z^{(t)} \gets b^{(t)}-\tilde{A}^{(t)}x_c^{(t)}$
    \EndIf
      \State $(x^{(t)}, \inv{G}^{(t)}, H^{(t)}) \gets \Call{Update}{a_t,\tilde{b}^{(t)}, \hat{x}_{c}^{(t)}, \tilde{A}^{(t)}, \inv{G}^{(t-1)}}$
  \Else
    \State $(\tilde{A}^{(t)},\tilde{b}^{(t)}) \gets (\tilde{A}^{(t-1)},\tilde{b}^{(t-1)})$
    \State $(x^{(t)},\inv{G}^{(t)},H^{(t)}) \gets (x^{(t-1)},\inv{G}^{(t-1)},H^{(t-1)})$
  \EndIf
  \State\Return $(x^{(t)}, \tilde{A}^{(t)}, \tilde{b}^{(t)}, \inv{G}^{(t)}, H^{(t)})$
\end{algorithmic}
\end{algorithm}

\begin{algorithm}[!ht]
\caption{$\algoname{Update}(a_t, \tilde{b}^{(t)}, \hat{x}_{c}^{(t)}, \tilde{A}^{(t)}, \inv{G}^{(t-1)})$}
\label{alg:update}
\begin{algorithmic}[1]
\State $g \gets a_t^{\top} \inv{G}^{(t-1)}a_t/p_t$
\State $\inv{G}^{(t)} \gets \inv{G}^{(t-1)}-\frac{1}{1+g}\inv{G}^{(t-1)}\frac{a_t a_t^{\top}}{p_t}\inv{G}^{(t-1)}$
\State $J^{(t)} \gets $ updated JL matrix after adding a new independent column
\State $F^{(t)} \gets J^{(t)}\tilde{A}^{(t)}$
\State $H^{(t)} \gets F^{(t)}\inv{G}^{(t)}$
\If{$\tilde{b}^{(t)} = \NULL$}
\State \Return $(\inv{G}^{(t)}, H^{(t)})$
\ElsIf{$\hat{x}_{c}^{(t)} = \NULL$}
\State $x^{(t)} \gets \inv{G}^{(t)}\tilde{A}^{(t)\top}\tilde{b}^{(t)}$
\Else
\State $x^{(t)} \gets \inv{G}^{(t)}\tilde{A}^{(t)\top}(\tilde{b}^{(t)}-\tilde{A}^{(t)}\hat{x}_{c}^{(t)})$
\EndIf
\State\Return $(x^{(t)}, \inv{G}^{(t)}, H^{(t)})$
\end{algorithmic}
\end{algorithm}

As mentioned in the preceding subsection, it is computationally expensive to compute Lewis weights in general. A special case is $p=2$, where the Lewis weights are leverage scores and are much easier to compute. In this case, $w_i(A) = a_i^\top (A^\top A)^{-1} a_i$, and correspondingly, the online Lewis weights become online leverage scores, which are $w_i^{\textrm{OL}}(A) = a_i^\top ((A^{(i)})^\top A^{(i)})^{-1} a_i$. It is much easier to compute $w_i^{\textrm{OL}}(A)$ in the online setting because one can simply maintain $(A^{(i)})^\top A^{(i)}$ by adding $a_i a_i^\top$ when reading a new row $a_i$ (viewed as a column vector). A na\"ive implementation of this algorithm would require inverting a $d\times d$ matrix at each step and we can further optimize the running time by noticing that $((A^{(i)})^\top A^{(i)})^{-1}$ 
receives a rank-one update at each step. This is the approach taken by~\cite{CJN18} and~\cite{JPW} for computing the online leverage scores in the online setting. Adopting this approach, we present our fast algorithm for $p=2$ in Algorithm~\ref{alg:online-active-regression} and its guarantee below. 

\begin{theorem}
\label{thm:p=2}
Let $A\in\R^{n\times d}$ and $b\in\R^n$. Assume that the minimum singular value of the first $d$ rows of $A$ is $\sigma > 0$. With probability at least $0.98-\delta$, Algorithm~\ref{alg:online-active-regression} makes 
\[
m = \cO\left(\frac{d}{\eps^{2}} \log^2 d \log^2 \frac{d}{\eps} \cdot \log\left(n \frac{\norm{A}_2}{\sigma}\right) \log \frac{1}{\delta} \right)
\]
queries in total and maintains for each $T=d+1,\dots,n$ a solution $\tilde{x}^{(T)}$ which satisfies that
\[
\Norm{A^{(T)}\tilde{x}^{(T)} - b^{(T)}}_2 \leq (1+\eps) \min_{x\in\R^d} \Norm{A^{(T)} x - b^{(T)}}_2.
\]
With probability at least $1-\delta$, Algorithm~\ref{alg:online-active-regression} runs in a total of
\[
\cO\left(\nnz(A)\log \frac{n}{\delta} + \frac{d^3}{\eps^4} \log n \log \frac{\norm{A}_2}{\sigma} \log \frac{1}{\eps \delta} \left( \log \frac{n}{\delta}+d \right) \right)
\]
time for processing the entire matrix $A$. Furthermore, with probability at least $1-\delta$, it uses $\cO(md)$ words of space in total.
\end{theorem}

\begin{remark}
The failure probability of Theorem~\ref{thm:p=2} can be reduced to $\delta$ by following the same approach as in Theorem~\ref{thm:p-compression-OAR-boost}. The query and space complexity remain asymptotically the same, but the runtime is increased to $\cO (\nnz(A)\log \frac{n}{\delta} + \frac{d^3}{\eps^4 \delta^4} \log n \log \frac{\norm{A}_2}{\sigma} \log \frac{1}{\eps \delta} ( \log \frac{n}{\delta}+d ) )$ because of independent copies of $S_1$.
\end{remark}

\begin{remark}
In comparison to~\cite{JPW}, our algorithm only requires a constant-factor approximation from the Johnson-Lindenstrauss matrix, saving a $1/\eps^2$ factor for the $\nnz(A)$ term in the runtime. In contrast, \cite{JPW} generate a new Johnson-Lindenstrauss matrix every time for robustness against adversarial attacks.
\end{remark}

In addition to the fast runtime, a major benefit of Algorithm~\ref{alg:online-active-regression} over the previous Algorithm~\ref{alg:p-simple-sample} is that we can now output a guaranteed $(1+\eps)$-approximation solution $\hat{x}^{(t)}$ efficiently in all intermediate steps.  instead of the outputting $\hat{x}^{(t)}$ only at the end. It is possible to do the same in Algorithm~\ref{alg:p-simple-OAR} for the general $p$, however, solving a general $\ell_p$ regression is computationally expensive and so we do not pursue maintaining the solution throughout the process. We also remark that the dependence on the online condition number of $A$ in Theorems~\ref{thm:p-naive-OAR} and~\ref{thm:p-compression-OAR} is improved to $\log(\norm{A}_2/\sigma)\leq \log\kappa^{\OL}(A)$.

\subsection{The case \texorpdfstring{$p=1$}{Lg}}

\begin{algorithm}[tb]
\caption{Online Active Regression for $p=1$}
\label{alg:1-online-active-regression}
\textbf{Initialize:} Let $\tilde{A}^{(d)}$ be the first $d$ rows of $A$ and $\tilde{b}^{(d)}$ be the first $d$ rows of $b$.
\begin{algorithmic}[1]
\State $\beta \gets \Theta(\log d)$
\State retain the first $d$ rows of $A$
\While{there is an additional row $a_t$}
\State $\tilde{w}_t \gets w_\last(\tilde{A}^{(t)})$
\State $p_{t} \gets \min(\beta \tilde{w}_{t} ,1)$
\State $(\tilde{A}^{(t)}, \tilde{b}^{(t)}) \gets \Call{Sample}{a_t, p_t, \tilde{A}^{(t-1)}, \tilde{b}^{(t-1)}, 1}$
\State $\tilde{x^{t}} \gets \Call{Reg}{\tilde{A}^{(t)}, \tilde{b}^{(t)}, 1}$
\EndWhile
\State \Return $\tilde{x}$
\end{algorithmic}
\end{algorithm}

The case of $p=1$ admits a simple sampling algorithm, based on~\cite{CP19} and \cite{PPP21}. We can simply sample the rows of $(A\ b)$ according to the online Lewis weights of $A$ without the multiple sampling schemes described at the beginning of the section. The algorithm is presented in Algorithm~\ref{alg:1-online-active-regression}, which can be implemented with the compression technique as described in Section~\ref{sec:p_between_1_2} for the approximations of online Lewis weights. The guarantee of the algorithm is then as follows.
\begin{theorem}\label{thm:p=1}
Let $A\in \R^{n\times d}$ and $b\in \R^{n}$. Algorithm~\ref{alg:1-online-active-regression} outputs a solution $\tilde{x}$ which satisfies that 
\[
\Norm{A\tilde{x}-b}_1\leq \min_{x\in\R^d} \Norm{Ax-b}_1
\]
with probability at least $1-\delta$ and makes 
\[
m = \cO\left(\frac{d}{\eps^{2}} \log \frac{d}{\eps \delta} \log n \log \kappa^{\OL}(A) \right)
\]
queries in total. When implemented with the compression technique, Algorithm~\ref{alg:1-online-active-regression} uses $\cO(md)$ words of space in total with probability at least $1-\delta$.
\end{theorem}
\section{Proofs of the Main Results}\label{sec:main-proof}
The framework of our Algorithms~\ref{alg:p-simple-OAR} and \ref{alg:online-active-regression} follows from the algorithm of~\cite{MMWY}. We first give a high-level idea of the proof in the offline case. Suppose that $z = b - Ax_c$ is the residual of a constant-factor approximation solution $x_c$ and $R = \min_x \norm{Ax-b}_p$ is the optimal error. Let $\mathcal{B}$ be an index set such that $\mathcal{B} = \{ i\in[n]: \frac{\abs{z_i}^p}{R^p} \geq \frac{w_i(A)}{\eps^p} \}$. Let $\Bar{z}$ be equal to $z$ but with all entries in $\mathcal{B}$ set to $0$. It can be shown that $\abs{\norm{S(z-\bar z)}_p^p - \norm{z-\bar z}_p^p} = \cO(\eps R^p)$ and the most arduous and difficult argument is to establish that $\abs{\norm{SAx-S\bar z}_p^p - \norm{Ax-\bar z}_p^p} = \cO(\eps R^p)$. 

In the online case, we can analogously define $\mathcal{B} = \{ i\in[n]: \frac{\abs{z_i}^p}{R^p} \geq \frac{w_i^{\OL}(A)}{\eps^p} \}$ using the online Lewis weights and we would still have $\abs{\norm{S(z-\bar z)}_p^p - \norm{z-\bar z}_p^p} = \cO(\eps R^p)$ in this case as the proof in~\cite{MMWY} still goes through. However, the key step, i.e.\ upper bounding $\abs{\norm{SAx-S\bar z}_p^p - \norm{Ax-\bar z}_p^p}$, requires considerable change, since our algorithm is a sampling algorithm in order to accommodate the online setting while Musco et al.'s algorithm is an iterative algorithm which reduces the dimension by a constant factor in each iteration and the proof of the error guarantee in each iteration cannot be easily ``flattened'' to fit a one-shot sampling algorithm. Therefore, we adopt the framework in~\cite{CP15} with  intermediate technical results in~\cite{MMWY} and prove the following error guarantee.

\begin{lemma}[Main lemma, informal] \label{lem:translation-preserve-informal}
Let $S\in \R^{r\times n}$ be the rescaled sampling matrix with respect to $\{ p_{i} \}_{(i)}$ such that $p_{i}=\min \{\beta\tilde{w}^{\OL}_{i}(A), 1\}$ for $\beta = \Omega(\frac{1}{\eps^2}\log^2 d \log n \log \frac{1}{\delta})$, it holds that
\[
\Pr \left \{ \max_{\norm{Ax}_p\leq R} \Abs{\norm{SA x-S \bar{z}}_p^p - \norm{Ax - \bar{z}}_p^p } \geq \eps R^p \right \} \leq \delta.
\]
\end{lemma}

The detailed statement and proof of Lemma~\ref{lem:translation-preserve-informal} will be given in Section~\ref{sec:omitted_proofs}. We note that for the case $p=2$, Theorem~\ref{thm:p=2} requires a different version of Lemma~\ref{lem:translation-preserve-informal} because the sampling matrices in Algorithm~\ref{alg:online-active-regression} do not have independent rows. The details are further postponed in Appendix~\ref{sec:p=2 appendix}.

In the next two subsections, we shall establish two basic results regarding online Lewis weights and their approximations, namely,
\begin{enumerate}[label=(\roman*)]
    \item \label{cod:1} the online $\ell_p$ Lewis weights calculated in  Algorithms~\ref{alg:p-simple-OAR} and~\ref{alg:online-active-regression} are within an absolute constant factor of the corresponding true $\ell_p$ online Lewis weights, and
    \item \label{cod:2} the sum of approximate $\ell_p$ online Lewis weights are bounded.
\end{enumerate}

When $p=1$, the framework of our Algorithm~\ref{alg:1-online-active-regression} simply adapts the $\ell_1$ Lewis weight sampling scheme in~\cite{CD21} and~\cite{PPP21} to online Lewis weight sampling. In order to prove Theorem~\ref{thm:p=1}, it suffices to show conditions~\ref{cod:1} and~\ref{cod:2} hold for $p=1$.

We remark that our proof relies on the fact that $w_i^{OL}(A)\geq w_i(A)$, which is guaranteed by the monotonicity of Lewis weights (Lemma~\ref{lem:monotonicity}) when $p\leq 2$. The monotonicity property does not always hold when $p>2$; thus, we only consider the case of $p\leq 2$ in this paper and leave the case of $p>2$ to future work.

\subsection{Sum of Online Lewis Weights}

Suppose that (i) holds, (ii) would follow from that the sum of true $\ell_p$ online Lewis weights are bounded, which are exactly the following two lemmas, for $p\in [1,2)$ and $p=2$, respectively.
\begin{lemma}\label{lem:sum_of_online_lewis_weights}
Let $p\in [1,2)$. It holds that $\sum_{i=1}^n w_i^{\textrm{OL}}(A) = \mathcal{O}(d\log n \cdot \log \kappa^{\OL}(A))$.
\end{lemma}

\begin{lemma}[{\cite[Lemma 2.2]{CMP}}]
\label{lem:sum-approx-online-2LW}
Let $p=2$. Suppose that the first $d$ rows of $A$ has the smallest singular value $\sigma>0$. It holds that
$\sum_{i=1}^n w_i^{\textrm{OL}}(A) = \cO(d\log(\norm{A}_2/\sigma))$.
\end{lemma}

The case $p=1$ of Lemma~\ref{lem:sum_of_online_lewis_weights} appeared in~\cite{sliding_window}. We  generalize the result to $p\in (1,2)$, following their approach. The proof can be found in Appendix~\ref{sec:online_lewis_weights}.

In the analysis of Algorithm~\ref{alg:p-simple-OAR}, we shall apply Lemma~\ref{lem:sum_of_online_lewis_weights} to $\tilde{A}_{1} = S_1 A$, where $S_1$ is a rescaled sampling matrix w.r.t.\ the online Lewis weights of $A$. To upper bound $\kappa^{\OL}(S_1A)$, we shall need the following auxiliary lemma, whose proof is postponed to Appendix~\ref{sec:proof-sum-sub-online-LW}.

\begin{lemma}
\label{lem:sum-sub-online-LW}
Let $p \in [1,2)$ and $S$ is a rescaled sampling matrix w.r.t.\ the online Lewis weights of $A$ and the oversampling parameter $\beta$. With probability at least $1-\delta$, it holds that $\log \kappa^{\OL}(SA) = O(\log(n\kappa^{\OL}(A) /(\beta\delta)))$.
\end{lemma}

\subsection{Approximating Online Lewis Weights}\label{sec:approx_online_lewis_weights_proof}

Now, it remains to prove (i) in order to prove the guarantee of $\tilde{x}$ in Theorems~\ref{thm:p=2} and \ref{thm:p-compression-OAR}.

First, the guarantee of approximate $\ell_2$ online Lewis weights follows from the works of \cite{CMP} and \cite{JPW}, which we cite below.

\begin{lemma}[{\cite[Theorem 2.3]{CMP}}, {\cite[Lemma 3.4]{JPW}}]\label{lem:spc-approx}
Let $\{\tilde{w}_i\}_i$ be the approximate Lewis weights in Algorithm \ref{alg:online-active-regression} and $\beta= \Theta(\eps^{-2}\log(d/\delta) )$. Let $S$ be the rescaled sampling matrix with respect to $\{\tilde{w}_i\}_i$. It holds with probability at least $1-\delta$ that 
\[
(1-\eps)(A^{(t)})^{\!\top}\! A^{(t)}\preceq(SA^{(t)})^{\!\top}(SA^{(t)})\preceq (1+\eps)(A^{(t)})^{\!\top} \! A^{(t)}
\]
for all $t\in \{d+1,\dots,n\}$ and $S$ has $\cO(\beta \sum_{i=1}^n\tilde{w}_i )$ rows.
\end{lemma}
As a consequence, $\tilde{w_t}\geq \frac{1}{1+\eps}\cdot a_i^\top((A^{(t)})^\top A^{(t)})^{-1} a_i \geq (1-\eps) w_t^{\textrm{OL}}(A)$ for all $t\in \{d+1,\dots,n\}$. This establishes (i) when $p=2$.

The case of general $p$ follows from Theorem~\ref{thm:compression-matrix}.
The following lemma is the key to the proof. 

\begin{lemma}\label{lem:compression-online-lewis-weights}
Let $A_i\in \R^{n_i\times d}$ $(i=1,\dots,r)$, $B\in \R^{k\times d}$ and $M = A_1 \circ A_2 \circ \cdots \circ A_r \circ B$. For each $i\in [r]$, let $S_i\in \R^{m_i\times n_i}$ be the rescaled sampling matrix with respect to $p_{i,1},\dots,p_{i,n_i}$ with $\min\{\beta w_j(A_i),1\} \leq p_{i,j}\leq 1$ for each $j\in [n_i]$, where $\beta = \Theta(\eta^{-2}\log (d/\delta))$. Let $M' = S_1 A_1 \circ \cdots \circ S_r A_r \circ B$. Then, with probability at least $1-\delta$, it holds for all $j=1,\dots,k$ that
\[
(1-\eta)w_{n_1+\cdots+n_r+j}(M) \leq w_{m_1+\cdots+m_r+j}(M') 
\leq (1+\eta) w_{n_1+\cdots+n_r+j}(M).
\]
\end{lemma}
A full version of the preceding lemma and its proof are postponed to Lemma~\ref{lem:compression_lewis_weights_appendix}.
Now we turn to prove Theorem~\ref{thm:compression-matrix}.
\begin{proof}\textbf{of Theorem~\ref{thm:compression-matrix} }
Observe that each block $B_i$ is the compressed version of $2^i$ smaller matrices, say, $A_1,\dots,A_{2^i}$, and each smaller matrix is compressed at most $i$ times. The compression scheme inside $B_i$ can be represented by a tree $T_i$, which satisfies that the root of $T_i$ has $i$ children  $T_{i-1},T_{i-2},\dots,T_0$. Every internal node of the tree represents a compression operation, which subsamples (with rescaling) the vertical concatenation of its children. An illustration of $T_i$ is shown in Figure~\ref{fig:tree}.

\begin{figure}
\centering
\begin{tikzpicture}[scale=0.8]
\node[circle,draw,inner sep =0.05em](A1) at (0,0) {$A_{1}$};
\node[circle,draw,inner sep =0.05em](A2) at (1,1) {$A_{2}$};
\node[circle,draw,inner sep =0.05em](A3) at (2,1) {$A_{3}$};
\node[circle,draw,inner sep =0.05em](A4) at (3,2) {$A_{4}$};
\node[circle,draw,inner sep =0.5em](S1)at (0,1) {};
\node[circle,draw,inner sep=0.5em](S2) at (0.5,2) {};
\node[circle,draw,inner sep=0.5em](S3) at (2,2) {};
\node[circle,draw,inner sep=0.5em](S4) at (1.5,3) {};
\node[circle,draw,inner sep=0.5em](S5) at (1.5,4) {};
\node[](T1) at (3,0) {$T_{i-1}$};

\draw[-] (A1)--(S1);
\draw[-] (S1)--(S2);
\draw[-] (A2)--(S2);
\draw[-] (S3)--(A3);
\draw[-] (S4)--(S2);
\draw[-] (S4)--(S3);
\draw[-] (S4)--(A4);
\draw[dotted,thick] (S5)--(S4);
\draw[dashed] (-0.7,-0.5) to [bend left =40] (S5);
\draw[dashed] (3.7,-0.5) to [bend right =40] (S5);
\draw[dashed] (-0.7,-0.5)--(3.7,-0.5);
\node[circle,draw,inner sep=0.05em](A5) at (4.6,1.6) {$A_{5}$};
\node[circle,draw,inner sep=0.5em](S6) at (4.6,2.6) {};
\node[circle,draw,inner sep=0.05em](A6) at (5.6,2.6) {$A_{6}$};
\node[circle,draw,inner sep=0.5em](S7) at (5.1,3.6) {};
\node[circle,draw,inner sep=0.5em](S8) at (5.1,4.6) {};
\node[](T2) at (5.7,1.4) {$T_{i-2}$};

\draw[-] (A5)--(S6);
\draw[-] (S6)--(S7);
\draw[-] (A6)--(S7);
\draw[dotted,thick] (S7)--(S8);
\draw[dashed] (4.2,1.1) to [bend left =25] (S8);
\draw[dashed] (6.2,1.1) to [bend right =25] (S8);
\draw[dashed] (4.2,1.1) to (6.2,1.1);

\node[circle,draw,inner sep=0.02em,font=\footnotesize](A7) at (7.5,3.5) {$A_{\scriptscriptstyle{2^i\!-\!1}}$}; 
\node[circle,draw, inner sep=0.5em](S9) at (7.5,4.5) {};
\node[](T3) at (7.8,2.6) {$T_1$};

\draw[-] (A7)--(S9);
\draw[dashed] (6.8,2.9) to (8.2,2.9);
\draw[dashed] (6.8,2.9) to [bend left =20] (S9.west);
\draw[dashed] (8.2,2.9) to [bend right=20] (S9.east);

\node[circle,draw,inner sep=0.04em](A9) at (8.7,4.5) {$A_{\scriptscriptstyle{2^i}}$};
\node[](T0) at (8.75,3.8) {$T_{0}$};

\node[circle,draw,inner sep=0.5em](Bi) at (4,6.1) {};
\draw[-] (S5)--(Bi);
\draw[-] (S8)--(Bi);
\draw[-] (S9)--(Bi);
\draw[-] (A9)--(Bi);

\node[](D) at (6.3,4.5) {$\cdots$};

\end{tikzpicture}
\caption{Tree structure of $T_i$ for block $B_i$}\label{fig:tree}
\end{figure}
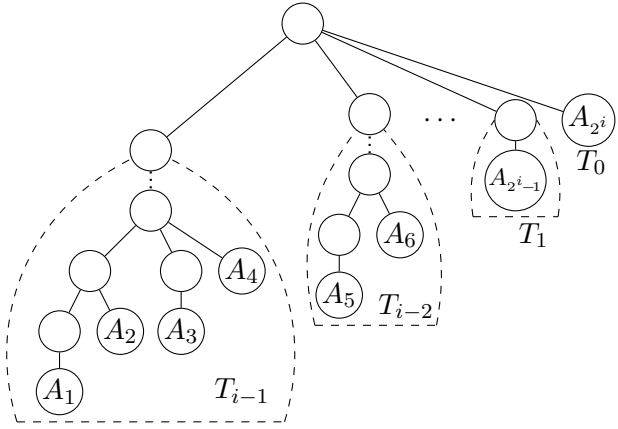

Consider a decompression process which begins at the root and goes down the tree level by level. When going down a level, we decompress each internal node on that level into the vertical concatenation of its children. When the decompression process is completed, we will have a vertical concatenation of the leaves, namely, $A_1\circ A_2\circ \cdots \circ A_{2^i}$, which is a submatrix of $A^{(t)}$.

Let $i^\ast$ be the largest $i$ such that $B_i$ is nonempty. Consider the decompression process of all blocks $B_{\log n}\circ \cdots \circ B_0$. This process will terminate in $i^\ast$ steps,
\[
A^{(t,i^\ast)} \to A^{(t,i^\ast-1)} \to \cdots \to A^{(t,0)},
\]
where $A^{(t,i^\ast)} = B_{\log n}\circ \cdots \circ B_0$ and $A^{(t,0)} = A^{(t)}$. Let $\tilde{w}_{t,j} = w_{\last}(A^{(t,j)})$. Note that $\tilde{w}_{t,0} = w_t(A^{(t)})$. By Lemma~\ref{lem:compression-online-lewis-weights} and our choices of parameters, it holds that
\[
 \left(1-\frac{\eta}{2\log n}\right)\tilde{w}_{t,j} \leq \tilde{w}_{t,j+1}\leq \left(1+\frac{\eta}{2\log n}\right)\tilde{w}_{t,j}
\]
with probability at least $1-\delta/\poly(n)$.  Iterating yields that
\[
 \left(\!1\!-\!\frac{\eta}{2\log n}\!\right)^{\!i\ast} \! w_{t}(A^{(t)}) \leq \tilde{w}_{t,i^\ast} 
 \leq \left(\!1\!+\!\frac{\eta}{2\log n}\!\right)^{\!i^\ast} \! w_{t}(A^{(t)}). 
\]
Note that $\tilde{w}_{t,i^\ast} = \tilde{w}_t$ per~\eqref{eqn:compressed_update}. Since $i^\ast\leq \log n$, we have
\[
(1-\eta)w_{t}(A^{(t)}) \leq \tilde{w}_{t,i^\ast} \leq (1+\eta)w_{t}(A^{(t)}).
\]
Taking a union bound over all $t$ gives the claimed result.
\end{proof}

\subsection{Proof of Theorem~\ref{thm:p-compression-OAR}} \label{sec:omitted_proofs}
We may assume that $n > \frac{d}{\eps^2} (\poly(\log d \log n \log \kappa^{\OL}) \log \frac{1}{\delta}$, otherwise it will not be necessary to sample for solving the regression problem.

The main lemma we shall prove is Lemma~\ref{lem:translation-preserve}. Before proving it, we state a series of lemmas, namely Lemmas~\ref{lem:constant-factor-approximation}, \ref{lem:3.6} and \ref{lem:3.7}, which, together with Lemma~\ref{lem:translation-preserve}, will prove Theorem~\ref{thm:p-compression-OAR}. We do not repeat the proof of  Lemmas~\ref{lem:constant-factor-approximation}, \ref{lem:3.6} and \ref{lem:3.7} because they are almost identical to those in~\cite{MMWY}, except that Lewis weights are replaced with online Lewis weights, which does not affect all these proofs since it is always true that $w_i^{\OL}(A)\geq w_i(A)$.

\begin{lemma}[Theorem 3.2 in {\cite{MMWY}}]
\label{lem:constant-factor-approximation}
Let $A\in\R^{n\times d}$, $b\in\R^n$, $p \in [1,2]$. If we sample $A$ and obtain $x_c$ by Algorithm~\ref{alg:p-simple-OAR} or Algorithm~\ref{alg:online-active-regression} with $\beta=\Theta(\log (d/\delta))$ then with probability at least $1-\delta$, 
\[
\Norm{Ax_c-b}_p \leq 2^{1+\frac{1}{p}} 3/\delta ^{\frac{1}{p}} R.
\]
\end{lemma}

In the remainder of this subsection, we define $z = b - Ax_c$ with a constant $\delta$ in Lemma~\ref{lem:constant-factor-approximation} and $R = \norm{z}_p$, then $R \leq C\min_{x\in\R^d} \norm{Ax-b}_p$ for some constant $C>0$.

\begin{lemma}[Lemma 3.6 in {\cite{MMWY}}] 
\label{lem:3.6}
Let $\mathcal{B}$ be an index set such that $\mathcal{B} = \{ i\in[n]: \frac{\abs{z_i}^p}{R^p} \geq \frac{w^{\OL}_i(A)}{\eps^p} \}$. Let $\Bar{z}$ be equal to $z$ but with all entries in $\mathcal{B}$ set to $0$. Then for all $x\in\R^d$ with $\norm{Ax}_p \leq R$,
$$\Abs{\Norm{Ax-z}_p^p- \Norm{Ax-\bar{z}}_p^p- \Norm{z-\bar{z}}_p^p}= \cO(\eps)R^p.$$
\end{lemma}

\begin{lemma}[Lemma 3.7 in {\cite{MMWY}}] 
\label{lem:3.7}
Consider the same setting in Lemma~\ref{lem:3.6}. 
Let $S$ be a sampling matrix according to the online Lewis weights with oversampling parameter $\beta = \Omega(\log d)$. 
With probability at least $0.995$, $\norm{Sz}_p^p=\cO(R^p)$ and for any $x\in\R^d$ with $\norm{Ax}_p\leq R$, it holds that 
$$\Abs{\Norm{SA x - Sz}_p^p - \Norm{SAx - S\bar{z}}_p^p - \Norm{Sz-S\bar{z}}_p^p} = \cO(\eps)R^p.$$
\end{lemma}

Next, we turn to prove the main lemma, which we restate below. 

\begin{lemma}[Main lemma]
\label{lem:translation-preserve}
Consider the same setting in Lemma~\ref{lem:3.6}. Let $S\in \R^{r\times n}$ be the rescaled sampling matrix with respect to $\{ p_{i} \}_{(i)}$ such that $p_{i}=\min \{\beta\tilde{w}^{\OL}_{i}, 1 \}$ and $\beta = \Theta(\frac{1}{\eps^2} \log^2 d \log n \log \frac{1}{\delta})$, it holds that
\[
\Pr \left \{ \max_{\Norm{Ax}_p\leq R} \Abs{\Norm{SA x-S \bar{z}}_p^p - \Norm{Ax - \bar{z}}_p^p } \geq \eps R^p \right \} \leq \delta.
\]
\end{lemma}

For notation simplicity, given $A\in \R^{n\times d}$ and $\eps,R > 0$, we say $z\in \R^n$ conforms to $(A,\eps,R)$ if $|z_i|^p \leq (R/\eps)^p w_i^{\OL}(A)$.

The next lemma is the generalization of Lemma 3.8 in \cite{MMWY}; here we give the $\ell$-th moment bound.
\begin{lemma}[Online version of {\cite[Lemma 3.8]{MMWY}}]\label{lem:uniform-moment-bound}
Let $p\in [1,2]$ and $\eps,R > 0$. Suppose that $A \in \R^{n \times d}$ with Lewis weights bounded by $\cO(\frac{d W }{n})$, where $W$ satisfies that $n = \Omega(\frac{d}{\eps^2}W\log^2 d \log n \log\frac{1}{\delta})$, and $\bar z\in \R^n$ conforms to $(A,\eps,R)$. Let 
\[
\Lambda=\max_{\norm{Ax}_p \leq R} \Abs{\sum_{k=1}^{n} \sigma_{k} \Abs{ a_k^{\top}x - \bar{z}_k }^p},
\] 
where $\sigma_{k}$'s are independent Rademacher variables, then it holds for $\ell = \log(1/\delta)$ that
\[
\E_{\sigma} \left[ \Lambda^\ell \right] \leq (\eps R^p)^{\ell} \delta.
\]
\end{lemma}

Lemma 3.8 in \cite{MMWY} proves that $\Lambda$ has a subgaussian tail. Using the following property of the subgaussian variable, we can obtain $\Lambda$'s $\ell$-th moment bound.

\begin{proposition}\label{prop:subgaussian}
Let $X$ be a subgaussian variable such that $\Pr\{ |X| > z\} \leq C\exp(-c z^2)$ for some constants $C, c > 0$ and every $z>0$. Then
$
\left( \E |X|^{\ell} \right)^{\frac{1}{\ell}} \leq K\sqrt{\ell}
$ for all $\ell \geq 1$, where $K$ is a constant that depends on $C$ and $c$ only.
\end{proposition}

\begin{proof}\textbf{of Lemma~\ref{lem:uniform-moment-bound}}
We will follow the approach in the proof of \cite[Lemma 3.8]{MMWY} and only highlight the changes. Their lemma assumes that the upper bound of Lewis weights are $\cO(d/n)$ and we shall modify it to $\cO(dW/n)$. This upper bound on Lewis weights was used in  \cite[Equation (5)]{MMWY}. Changing it to $\cO(dW/n)$ yields, via the same Dudley's integral as in the proof of \cite[Lemma 3.8]{MMWY}, that
\[
\E_{\sigma} \Lambda  \leq C' \left( \frac{dW}{n} \log^2 d \log n\right)^{\frac{1}{2}} R^p,
\]
where $C'>0$ is an absolute constant. The tail-bound version of Dudley's integral then implies that
\[
\Pr\left\{\Lambda\geq C\sqrt{\frac{d}{n}} \left(\log d \sqrt{W \log n} + z \right) R^p \right\} \leq 2\exp(-z^2),
\]
where $C > 0$ is an absolute constant.
Hence, let $L=\log d \sqrt{W\log n}$ and we have $\Pr\{ \frac{\Lambda}{CR^{p} \sqrt{d/n}}-L >z \} \leq 2\exp(-z^2)$. The variable $\frac{\Lambda}{CR^p \sqrt{d/n}}-L$ has a subgaussian tail. It follows from Proposition~\ref{prop:subgaussian} that
$\E_{\sigma} \left(\frac{\Lambda}{CR^p \sqrt{d/n}} \right)^{\ell} \leq (L + K\sqrt{\ell})^{\ell}$.
Therefore, we have 
\[
\E_{\sigma} \Lambda^{\ell} \leq \left( \frac{L + K \sqrt{\ell}}{\sqrt{\eps^2 n/d}} \right)^{\ell} (\eps R^p)^{\ell} \leq (\eps R^p)^{\ell} \delta,
\]
where the second inequality is from our assumption that $n > \frac{e^2 d}{\eps^2} (L + K\sqrt{\ell})^2$, provided that $\ell = \log\frac{1}{\delta}$.
\end{proof}

The next lemma is an analogous version to Lemma 7.4 in \cite{CP15}, adapted to the online active regression setting.

\begin{lemma}
\label{lem:Ax-z-preserve}
Suppose there exists $\ell\geq 1$ such that whenever a matrix $A = a_{1} \circ a_{2} \circ \dots \circ a_{n} \in \R^{n\times d}$ has Lewis weights uniformly bounded by $\cO(d W/n)$, where $W$ satisfies that $n = \Omega(\frac{d}{\eps^2} (W \log^2 d \log n \log \frac{1}{\delta} ))$,
it holds for all $R > 0$ and $\bar{z} \in \R^n$ conforming to $(A,\eps,R)$ that
\[
\E_{\sigma}\left[ \left( \max_{\norm{Ax}_p \leq R} \Abs{ \sum_{k=1}^{n} \sigma_{k} \Abs{ a_k^{\top}x - \bar{z}_k }^p } \right)^{\ell} \right] \leq (\eps R^p)^{\ell} \delta.
\] 
Then, let $A \in \R^{n \times d}$, $\bar{z} \in\R^n$ be as defined in Lemma~\ref{lem:3.6}
and $S$ be the rescaled sampling matrix with respect to the online Lewis weights of A with  oversampling parameter $\beta = \Theta(\frac{1}{\eps^{2}} \log^2 d \log n \log \frac{1}{\delta})$. 
With probability at least $1-\delta$, it holds that
\begin{enumerate}[label=(\roman*)]
\item the number of rows in $S$ is
\[
\cO\left(\frac{d}{\eps^{2}} \log^2 d \log^2 n \log \kappa^{\OL}(A) \log \frac{1}{\delta} \right)
\]
and
\item
\[
\max_{\Norm{Ax}_p\leq R} \Abs{\Norm{SA x-S \bar{z}}_p^p - \Norm{Ax - \bar{z}}_p^p } \leq \eps R^p.
\]
\end{enumerate}
\end{lemma}

\begin{proof}
Note that the sampling probability $p_i = \min \{ \beta w_i, 1 \}$. If $\beta w_{i}>1$, we have $p_i = 1$ and hence $S(Ax-\bar{z})_i = (Ax-\bar{z})_i$. Therefore, we only consider the case $\beta w_i \leq 1$. Let
\[
M=  \left( \max_{\norm{Ax}_p \leq R} \Abs{\norm{SAx - S\bar{z}}_p^p - \norm{Ax-\bar{z}}_p^p} \right)^\ell.
\] 
Since taking the $\ell$-th power of a maximum is convex, we have 
\begin{align*}
\E_{S} M&= \E_{S}\left[ \left( \max_{\norm{Ax}_p \leq R 
} \Abs{\norm{SAx - S\bar{z}}_p^p  - \E_{S'} \norm{S'Ax - S'\bar{z}}_p^p }  \right)^{\ell} \right] \\
&\leq 
\E_{S,S'}\left[ \left( \max_{\norm{Ax}_p \leq R  
} \Abs{ ( \norm{SAx - S\bar{z}}_p^p -  \norm{S'Ax - S'\bar{z}}_p^p } \right)^{\ell} \right],
\end{align*}
where $S'$ and $S$ are independent and identically distributed rescaled sampling matrices w.r.t.\ Lewis weights. Therefore, we have
\begin{align*}
\norm{SAx - S\bar{z}}_p^p -  \norm{S'Ax - S'\bar{z}}_p^p = \sum_{k=1}^{n} \frac{(\mathbbm{1}_{S})_{k}}{p_{k}} \Abs{a_{k}^{\top} x - \bar{z}_{k}}^p - \sum_{k=1}^{n} \frac{(\mathbbm{1}_{S'})_{k}}{p_{k}} \Abs{a_{k}^{\top} x - \bar{z}_{k}}^p
\end{align*}
Since $S'$ and $S$ are independent and have identical distributions, we have by symmetrization trick that
\[
\E_{S} M \leq 2^{\ell} \E_{S,\sigma} \left[ \left( \max_{\norm{Ax}_p \leq R} \Abs{\sum_{k=1}^{n} \frac{(\mathbbm{1}_{S})_{k}}{p_{k}} \sigma_k \Abs{a_{k}^{\top} x - \bar{z}_{k}}^p} \right)^{\ell} \right] =: 2^\ell \E_{S,\sigma} M'.
\]

It follows from Lemma~\ref{lem:lewis_weights_sampling} that $S$ is a $1/2$-subspace embedding matrix for $A$ with probability at least $1-\delta/3$. Furthermore, by Lemma~\ref{lem:compression_lewis_weights_appendix}, with probability at least $1-\delta/3$, the Lewis weights of the rows from $SA$ are within $[\frac{1}{2\beta},\frac{3}{2\beta}]$. Also by Chernoff bounds, $SA$ has $N =\Theta(\beta\sum_i w_i^{\OL}(A))$ rows with probability at least $1-\delta/3$. Let $\mathcal{E}$ denote the event on $S$ that these three conditions above hold. Then $\Pr(\mathcal{E})\geq 1-\delta$.

Next, fix $S\in\mathcal{E}$. Consider the conditional expectation
\[
\E_{\sigma} \left[ \left. M' \right\vert S\right] = \E_{\sigma}\left[ \left. \max_{\norm{Ax}_p \leq R} \Abs{\sum_{k=1}^{N} \sigma_k \Abs{a^{'\top}_{k} x - \bar{z}^{'}_{k}}^p}^\ell \right\vert S \right],
\]
where $a_k^{'\top}x$ and $\bar{z}^{'}_{k}$ are the $k$-th coordinates of $SAx$ and $S\bar{z}$. 

Recall that $S$ is a $1/2$-subspace embedding matrix for $A$ when conditioned on $\mathcal{E}$, we have for any $x \in \R^d$ that $\norm{SA x}_p \leq \frac{3}{2}\norm{Ax}_p \leq 2R$, which implies that  
\begin{align*}
    \max_{\norm{Ax}_p \leq   R} \Abs{\sum_{k=1}^{N} \sigma_k \Abs{a'^{\top}_{k} x - \bar{z}'_{k}}^p}^\ell \leq
    \max_{\norm{SA x}_p \leq 2R} \Abs{\sum_{k=1}^{N} \sigma_k \Abs{a'^{\top}_{k} x - \bar{z}'_{k}}^p}^\ell
\end{align*}

Now, we verify that $SA$ has small Lewis weights and $S\bar{z}$ conforms to $(SA, \eps, 2R)$. First, recall that the Lewis weights of $SA$ are within $[\frac{1}{2\beta},\frac{3}{2\beta}]$, where $\frac{1}{\beta} = \Theta(\frac{d \log n \log \kappa^{\OL}(A) }{N})$. Hence, $W=\log n \log \kappa^{\OL}(A)$ and $N =  \Omega(d\beta \log n\log \kappa^{\OL}(A)) = \Omega(\frac{d}{\eps^2} W \log^2 d \log n \log \frac{1}{\delta})$ as desired. Second, the coordinates $\abs{S \bar{z}_k}^p = \frac{|\bar{z}_k|^p}{\beta w_k(A)} \leq \frac{R^p}{\eps^p \beta} \leq 2(\frac{R}{\eps})^p w_{k}(SA) \leq (\frac{2 R}{\eps})^p w_{k}(SA)$. It then follows from the assumption of the lemma that
\[
\E_\sigma \left[ \left. M' \right\vert S\right] \leq \left(\eps 2^p R^p \right)^\ell \delta.
\]
Rescaling $\eps = \eps/2^{p+1}$ and taking expectation over $S$ while conditioned on $\mathcal{E}$, we have that
\[
\E_{S,\sigma} \left[ \left. M' \right\vert \mathcal{E}\right] \leq \left(\frac{\eps R^p}{2} \right)^\ell \delta.
\]
It then follows that
\[
\E_S \left[ \left. M \right\vert \mathcal{E}\right] \leq \left( \eps R^p \right)^\ell \delta
\]
and, by Markov's inequality,
\[
    \Pr \left \{ \left. \max_{\norm{Ax}_p\leq R} \Abs{\norm{SA x-S \bar{z}}_p^p - \norm{Ax - \bar{z}}_p^p } \geq \eps R^p \right\vert \mathcal{E} \right \} 
    \leq \Pr\left\{ \left. M \geq (\eps R^p)^\ell \right\vert \mathcal{E} \right\}
\leq \delta.
\]
Rescaling $\delta = \delta/2$ for a union bound completes the proof.
\end{proof}


Finally, Lemma~\ref{lem:translation-preserve} is now immediate by combining Lemmas~\ref{lem:uniform-moment-bound} and~\ref{lem:Ax-z-preserve}. 
Combining the results of Lemmas~\ref{lem:3.6}, \ref{lem:3.7} and \ref{lem:Ax-z-preserve} and following the proof in \cite[Lemma 3.9]{MMWY} gives the following guarantee.

\begin{lemma}\label{lem:Ax-z-optimal-approximation}
    Consider the same setting in Lemma~\ref{lem:constant-factor-approximation} and let $z = b - Ax_c$. Suppose that $S$ is a sampling matrix according to the online Lewis weights of $A$ with oversampling factor $\beta = \Theta(\frac{\log^2 d}{\eps^2} \log n \log\frac{1}{\delta})$ and $\bar{x} = \arg \min_{x \in \R^d} \norm{S Ax - S z}_p$. 
    It holds with probability at least $0.99-\cO(\delta)$ that
    \[
        \norm{A\bar{x} - z}_p^p \leq \left(1 + \cO(\eps)\right)\min_{x\in\R^d} \norm{Ax-z}_p^p.
    \]
\end{lemma}

We are now ready to prove Theorem~\ref{thm:p-compression-OAR}.

\begin{proof}\textbf{of Theorem~\ref{thm:p-compression-OAR}}
Recall that we can write $\tilde{A}_1 = S_1 A$ for a sampling matrix $S_1$ with respect to the online Lewis weights of oversampling parameter $\beta_1 = \Theta(\eps^{-(2+p)} d\log\frac{1}{\eps\delta} )$.
It follows from Lemmas~\ref{lem:constant-factor-approximation} and~\ref{lem:Ax-z-optimal-approximation} that with probability at least $0.99-\cO(\delta)$,
\[
\norm{S_1 A \hat{x} - S_1 z}_p^p \leq \left(1 + \cO(\eps)\right)\min_{x\in\R^d} \norm{S_1 A x - S_1 z}_p^p,
\]
where $z = b - A x_c$. 
Invoking an analogous argument of Lemma~\ref{lem:Ax-z-optimal-approximation} (following the identical approach in Section 3.5 of \cite{MMWY}), we can obtain that, with probability at least $0.995-\cO(\delta)$,
\[
\norm{A \hat{x} - z}_p^p \leq \left(1 + \cO(\eps)\right)\min_{x\in\R^d} \norm{A x - z}_p^p.
\]
It then follows immediately that with probability at least $0.98-\cO(\delta)$,
\[
\norm{A \tilde{x} - b}_p^p \leq \left(1 + \cO(\eps)\right)\min_{x\in\R^d} \norm{A x - b}_p^p
\]

For the results above to go through, $S_1$ should have oversampling parameter $\beta_1 = \Theta(\frac{d}{\eps^{2+p}} \log\frac{1}{\eps\delta})$, resulting in 
\[
    N = \cO\left(\beta_1 \sum_{i=1}^n w_i^{\OL}(A)\right) = \cO\left(\frac{d^2}{\eps^{2+p}} \log\frac{1}{\eps\delta} \log n\log \kappa^{\OL}(A) \right)
\]
rows of $S_1A$ with probability at least $1-\delta$. Also, $S_3$ should have an oversampling parameter 
\[
    \beta_3 = \Theta\left( \frac{\log^2 d}{\eps^2}\log\frac{1}{\delta}\log N \right) = \cO\left( \frac{\log^2 d}{\eps^2}\log\frac{1}{\delta}\left(\log\frac{d}{\eps} + \log\log\frac{1}{\delta}\right) \right),
\]
resulting in
\begin{align*}
    m &= \cO\left(\beta_3 \sum_{i=1}^n w_i^{\OL}(SA)\right) \\
      &= \cO\left(\frac{d\log^2 d}{\eps^2} \log^2 N \log \frac{n\kappa^{\OL}(A)}{\beta_3\delta} \log\frac{1}{\delta} \right) \\
      &= \cO\left(\frac{d\log^2 d}{\eps^2} \left( \log\frac{d}{\eps} + \log\log\frac{1}{\delta} \right)^2 \log \frac{n\kappa^{\OL}(A)}{\delta} \log\frac{1}{\delta} \right) 
\end{align*}
rows of $S_3 S_1 A$ with probability at least $1-\delta$. Here we upper bound $\kappa^{\OL}(SA)$ by Lemma~\ref{lem:sum-sub-online-LW}.

The total number of queried labels is dominated by $m$. Rescaling $\eps$ and $\delta$ gives the claimed result.
\end{proof}

\subsection{Time Complexity for \texorpdfstring{$p=2$}{p=2}}

\begin{lemma}\label{lem:time-lewis}
With probability at least $1-\delta$, the running time of Algorithm~\ref{alg:online-active-regression} over $n$ iterations is $\cO(\nnz(A)\log \frac{n}{\delta} + \frac{d^3}{\eps^4} \log \frac{\norm{A}_2}{\sigma} \log\frac{1}{\eps \delta} (\log \frac{n}{\delta}+d))$.
\end{lemma}
\begin{proof}
We analyze the time complexity following Lemma 3.8 in \cite{JPW}. Note that total runtime is dominated by 
calls to $\Call{Update}{}$. The approximate Lewis weights are calculated by $\tilde{w}_t = \norm{H^{(t-1)}a_t}_2^2$, which takes $\cO(\nnz(A) \log \frac{n}{\delta})$ time over $n$ iterations. Observe that the runtime of each call to $\Call{Update}{}$ is dominated by the time calculating $F^{(t)}$ and $H^{(t)}$, which takes $\cO(d\log \frac{n}{\delta} + d^2)$ time. Calls to $\Call{Update}{}$ only happen when there is a new row $a_t$ is sampled and the number of samples is dominated by the maximum of the number of rows of $S$ and that of $S_1$, which with probability at least $1-\delta$ are $\cO(d\log d)$ and $\cO(\frac{d^2}{\eps^4} \log\frac{1}{\eps \delta} \log n \log \frac{\norm{A}_2}{\sigma})$, respectively. Hence, the total running time is $\cO(\nnz (A)\log \frac{n}{\delta} + \frac{d^4}{\eps^4} \log\frac{1}{\eps \delta} \log n \log \frac{\norm{A}_2}{\sigma} + \frac{d^3}{\eps^{4}} \log \frac{n}{\delta} \log\frac{1}{\eps \delta} \log n  \log \frac{\norm{A}_2}{\sigma})$.
\end{proof}

\section{Optimal dependence on \texorpdfstring{$\eps$}{1}} \label{sec:optimal_eps}

The query complexity in Theorem~\ref{thm:p-compression-OAR} has a quadratic dependence on $1/\eps$. In this section, we shall improve it to $1/\eps$, which is the optimal for active regression~\cite{MMWY} and is thus optimal for online active regression. Again, we follow the idea in~\cite{MMWY}, narrowing the region of $x$ in Lemma~\ref{lem:translation-preserve} to $\{x: \norm{Ax}_p \leq \sqrt{\gamma}R\}$, where $0<\gamma<1$. This will guarantee that $\norm{SAx-S\bar{z}}_p$ is close to $\norm{Ax-\bar{z}}_p$ for all $x$ near $x^\ast = \operatorname{argmin}_x \norm{Ax-\bar{z}}_p$, which is enough for approximately solving the original regression problem $\min_x \norm{Ax-b}_p$.

To prove the new version of Lemma~\ref{lem:translation-preserve},  we first need a new version of Lemma~\ref{lem:uniform-moment-bound}, which can be seen as an online version of \cite[Lemma 3.23]{MMWY} with an $\ell$-th moment bound.

\begin{lemma}
\label{lem:uniform-moment-bound-opt}
Let $p\in (1,2]$, $R > 0$ and $0 < \gamma < 1$. Suppose that $A \in \R^{n \times d}$ with Lewis weights bounded by $\cO(\frac{d W }{n})$, where $W$ satisfies that $n = \Omega(\frac{\gamma d}{\eps^2}W \log^2 d \log n \log\frac{1}{\delta})$, and $\bar z\in \R^n$ conforms to $(A,\eps,R)$. Let 
\[
    \Lambda = \max_{\norm{Ax}_p \leq \sqrt{\gamma}R} \Abs{\sum_{k=1}^{n} \sigma_{k} \left( \Abs{ a_k^{\top}x - \bar{z}_k }^p - \Abs{\bar{z}_k}^p \right)},
\] 
where $\sigma_{k}$'s are independent Rademacher variables, then it holds for $\ell = \log(1/\delta)$ that
\[
\E_{\sigma} \left[ \Lambda^\ell \right] \leq (\eps R^p)^{\ell} \delta.
\]
\end{lemma}

\begin{proof}
First, note that our $\bar{z}$ has the same definition as $b \in \R^n$ in \cite[Lemma 3.23]{MMWY}. Compared with \cite[Lemma 3.23]{MMWY}, the different condition is that the upper bound of the Lewis weights is $\cO(dW/n)$ and $W$ can depend on $A$. This modification is dealt with in the same manner as in the proof of our Lemma~\ref{lem:uniform-moment-bound}. 

Now it follows from the proof of \cite[Lemma 3.23]{MMWY}, with the modifications above, that
\[
    \E_{\sigma}\Lambda \leq C'\left(\frac{\gamma d}{n} W \log^2 d \log n\right)^{\frac{1}{2}}R^p
\]
and
\[
    \Pr\left\{\Lambda \geq C\frac{\gamma d}{n}\left(\log d \sqrt{W\log n} + z\right)R^p \right\} \leq 2\exp{(-z^2)}.
\]
Next, from our assumption that $n = \Omega(\frac{\gamma d}{\eps^2}W \log^2 d \log n \log\frac{1}{\delta})$, we can use the same approach of Lemma~\ref{lem:uniform-moment-bound} to obtain that
\[
\E_{\sigma} \left[ \Lambda^\ell \right] \leq (\eps R^p)^{\ell} \delta.
\]
\end{proof}

\begin{lemma}\label{lem:gamma-eps-guarantee}
Suppose there exists $\ell\geq 1$ such that whenever a matrix $A = a_{1} \circ a_{2} \circ \dots \circ a_{n} \in \R^{n\times d}$ has Lewis weights uniformly bounded by $\cO(d W/n)$, where $W$ satisfies that $n = \Omega(\frac{d}{\eps^2} (W \log^2 d \log n \log \frac{1}{\delta} ))$, 
it holds for all $R > 0$ and $\bar{z} \in \R^n$ conforming to $(A,\eps,R)$ that
\[
    \E_{\sigma}\left[ \left(\max_{\norm{Ax}_p \leq \sqrt{\gamma} R} \Abs{\sum_{k=1}^{n} \sigma_{k} \left(\Abs{ a_k^{\top}x - \bar{z}_k }^p - \Abs{\bar{z}_k}^p \right) } \right)^{\ell} \right] \leq (\eps R^p)^{\ell} \delta.
\] 
Then, let $A\in\R^{n\times d}$, $\bar{z} \in \R^n$ be as defined in Lemma~\ref{lem:3.6} and $S$ be the rescaled sampling matrix with respect to the online Lewis weights of $A$ and oversampling parameter $\beta = \Theta(\frac{\gamma}{\eps^2} \log^2 d \log n \log \frac{1}{\delta})$ where $0< \gamma <1$. It holds that
\[
    \max_{\norm{Ax}_p \leq \sqrt{\gamma} R} \Abs{\Norm{Ax - \bar{z}}_p^p -\Norm{SAx - S\bar{z}}_p^p - \left(\Norm{\bar{z}}_p^p - \Norm{S\bar{z}}_p^p\right)} = \cO(\eps) R^p  
\]
with probability at least $1 - \delta$.
\end{lemma}

\begin{proof}
The proof closely follows the framework in the proof of Lemma~\ref{lem:Ax-z-preserve}. First, as in the proof of Lemma~\ref{lem:Ax-z-preserve}, we only consider the case where 
$\beta w_i \leq 1$. Next, let
\[
M= \left( \max_{\Norm{Ax}_p \leq \sqrt{\gamma} R} \Abs{ \Norm{SAx - S\bar{z}}_p^p - \Norm{Ax-\bar{z}}_p^p + \Norm{\bar{z}}_p^p - \Norm{S\bar{z}}_p^p } \right)^\ell.
\] 
By the symmetrization trick, we have that
\[
\E_{S} M \leq 2^\ell \E_{S,\sigma} M', \text{ where } M' = \left( \max_{\norm{Ax}_p \leq \sqrt{\gamma} R} \Abs{\sum_{k=1}^{n} \frac{(\mathbbm{1}_{S})_{k}}{p_{k}} \sigma_k \left( \Abs{a_{k}^{\top} x - \bar{z}_{k}}^p - \Abs{\bar{z}_k}^p \right)} \right)^{\ell}.
\]

Let $\mathcal{E}$ be the same event on $S$ as defined in the proof of Lemma~\ref{lem:Ax-z-preserve}, then $\Pr(\mathcal{E})\geq 1-\delta$ and, when $\mathcal{E}$ happens, we have (i) $S$ is a $1/2$-subspace embedding matrix for $A$ (ii) the Lewis weights of the rows from $SA$ are within $[\frac{1}{2\beta},\frac{3}{2\beta}]$, and (iii) $SA$ has $N =\Theta(\beta\sum_i w_i^{\OL}(A))$ rows. 

Fix $S\in\mathcal{E}$. Consider the conditional expectation
\[
    \E_\sigma \left[ \left. M' \right\vert S \right] = \E_{\sigma}\left[ \left. \max_{\norm{Ax}_p \leq \sqrt{\gamma} R} \Abs{\sum_{k=1}^{N} \sigma_k \left( \Abs{a'^{\top}_{k} x - \bar{z}'_{k}}^p - \Abs{\bar{z}'_{k}}^p \right) } ^\ell \right\vert S \right],
\]
where $a_k'^{\top}x$ and $\bar{z}'_{k}$ are the $k$-th coordinates of $SAx$ and $S\bar{z}$. 

Recall that $S$ is a $1/2$-subspace embedding matrix for $A$ when conditioned on $\mathcal{E}$, we have for any $x \in \R^d$ that $\norm{SA x}_p \leq 3/2\norm{Ax}_p \leq   2 \sqrt{\gamma} R$, which implies that  
\begin{align*}
    \max_{\norm{Ax}_p \leq \sqrt{\gamma}  R} \Abs{\sum_{k=1}^{N} \sigma_k \left(\Abs{a'^{\top}_{k} x - \bar{z}'_{k}}^p - \Abs{\bar{z}'_k}^p \right) }^\ell
    \leq \max_{\norm{SA x}_p \leq 2  \sqrt{\gamma} R} \Abs{\sum_{k=1}^{N} \sigma_k \left( \Abs{a'^{\top}_{k} x - \bar{z}'_{k}}^p - \Abs{\bar{z}'_k}^p \right) }^\ell
\end{align*}

Similarly to the proof of Lemma~\ref{lem:Ax-z-preserve}, we can verify that $SA$ has small Lewis weights and $S\bar{z}$ conforms to $(SA, \eps, 2R)$.
It then follows from the assumption of the lemma that
\[
\E_\sigma \left[ \left. M' \right\vert S \right] \leq \left(\eps \cdot 2^{p} R^p \right)^\ell \delta.
\]
Rescaling $\eps = \eps/2^{p+1}$, we can obtain, in the identical manner as in the proof of Lemma~\ref{lem:Ax-z-preserve}, that
\[
\E_{S}\left[ \left. M \right\vert \mathcal{E} \right] \leq \left(\eps R^p \right)^\ell \delta
\]
and by Markov's inequality that
\begin{multline*}
    \Pr \left \{ \left. \max_{\norm{Ax}_p \leq \sqrt{\gamma} R} \Abs{\Norm{Ax - \bar{z}}_p^p -\Norm{SAx - S\bar{z}}_p^p - \left(\Norm{\bar{z}}_p^p - \Norm{S\bar{z}}_p^p\right)} \geq \eps R^p \right\vert \mathcal{E} \right \}  \\
    \leq \Pr\left\{ \left. M \geq (\eps R^p)^\ell \right\vert \mathcal{E} \right\}
\leq \delta.
\end{multline*}
Rescaling $\delta = \delta/2$ for a union bound completes the proof.
\end{proof}

The following lemma is the new version of Lemma~\ref{lem:translation-preserve}.

\begin{lemma}
\label{lem:main-gamma}
Let $A\in\R^{n\times d}$ and $\bar{z}$ be as defined in Lemma~\ref{lem:3.6}. Let $S$ be the rescaled sampling matrix with respect to $\{ p_{i} \}_{(i)}$ such that $p_{i}=\min \{\beta\tilde{w}^{\OL}_{i}(A), 1 \}$ and $\beta = \Theta(\frac{\gamma}{\eps^2} \log^2 d \log n \log \frac{1}{\delta})$. It holds that
\[
\Pr \left \{ \max_{\norm{Ax}_p \leq \sqrt{\gamma} R} \Abs{\Norm{SA x-S \bar{z}}_p^p - \Norm{Ax - \bar{z}}_p^p + \Norm{\bar{z}}_p^p - \Norm{S\bar{z}}_p^p } \geq \eps R^p \right \} \leq \delta.
\]
\end{lemma}

\begin{proof}
Combining Lemmas~\ref{lem:uniform-moment-bound-opt} and \ref{lem:gamma-eps-guarantee} immediately gives the result.
\end{proof}

\begin{lemma}\label{lem:gamma-assumption-eps-approx}
Let $A \in \R^{n \times d}$, $b\in \R^n$ and $0< \gamma <1$. Let $S$ be as defined in Lemma~\ref{lem:main-gamma}. Let $x_c$ be the constant factor approximation obtained by Algorithm~\ref{alg:p-simple-OAR} or ~\ref{alg:online-active-regression} and $z = b - Ax_c$. Suppose that $\tilde{x} = \arg\min_{x\in \R^d} \norm{SA x - Sz}_p$ and $\norm{A \tilde{x} - Ax^*}_p \leq \sqrt{\gamma} R$. It holds that
\[
    \Norm{A\tilde{x} - z}_p \leq (1+\cO(\eps)) \min_{x\in\R^d} \Norm{Ax - z}_p
\]
with probability at least $0.99 - \cO(\delta)$.
\end{lemma}

\begin{proof}
Let $x' = x -x^*$ and $z' = z + Ax^*$. We have that
\begin{align*}
    \norm{A\tilde{x}-z}_p^p - \norm{Ax^*-z}_p^p &= \norm{A\tilde{x}-z}_p^p - \norm{SA\tilde{x} - Sz}_p^p + \norm{SA\tilde{x} - Sz}_p^p - \norm{SAx^* - Sz}_p^p\\
    &\qquad \qquad + \norm{SAx^* - Sz}_p^p - \norm{Ax^* - z}_p^p\\
    &\leq \norm{A\tilde{x}-z}_p^p - \norm{SA\tilde{x} - Sz}_p^p + \norm{SAx^* - Sz}_p^p - \norm{Ax^* - z}_p^p\\
    &= \norm{Ax' - z'}_p^p - \norm{SAx' - Sz'}_p^p + \norm{Sz'}_p^p - \norm{z'}_p^p.
\end{align*}
Here, it holds that $\norm{x'}_p \leq \sqrt{\gamma}R$ and $\norm{z'}_p \leq \norm{2z}_p + \norm{Ax^* - z}_p \leq C' R$. Then,
\begin{align*}
&\quad\, \Norm{Ax' - z'}_p^p - \Norm{SAx' - Sz'}_p^p + \Norm{Sz'}_p^p - \Norm{z'}_p^p \\
&= \Norm{Ax' - z'}_p^p - \Norm{SAx' - Sz'}_p^p + \Norm{S\bar{z}'}_p^p + \Norm{Sz' - S\bar{z}'}_p^p - \Norm{\bar{z}'}_p^p - \Norm{z' - \bar{z}'}_p^p\\
&= \Norm{Ax' - z'}_p^p - \Norm{Ax' - \bar{z}'}_p^p - \Norm{z' - \bar{z}'}_p^p \\
&\qquad - \left( \Norm{SAx' - Sz'}_p^p - \Norm{SAx' - S\bar{z}'}_p^p - \Norm{Sz' - S\bar{z}'}_p^p \right) \\
&\qquad - \left( \Norm{SA x'-S \bar{z}'}_p^p - \Norm{Ax' - \bar{z}'}_p^p + \Norm{\bar{z}'}_p^p - \Norm{S\bar{z}'}_p^p \right) \\
&= \cO(\eps)R^p
\end{align*}
with probability at least $0.99 - \cO(\delta)$, where the last line follows from Lemmas~\ref{lem:3.6}, \ref{lem:3.7} and \ref{lem:gamma-assumption-eps-approx} (with $x=x'$ and $z=z'$).
It then follows that
\[
    \Norm{A\tilde{x} - z}_p \leq (1+\cO(\eps)) \min_{x\in\R^d} \Norm{Ax - z}_p
\]
with probability at least $0.99-\cO(\delta)$.
\end{proof}

Now, we explain why it suffices to consider only $x$ satisfying that $\norm{Ax}_p \leq \sqrt{\gamma}R$. First, we prove all good approximate solutions are near to the optimal solution.

\begin{lemma}
\label{lem:gamma-correctness}
Let $A \in \R^{n\times d}$, $z \in \R^n$ and $0<\gamma<1$. We assume that $\norm{z}_p \leq R$ where $R = \min_{x\in \R^d} \norm{Ax - z}_p$. Let $x^* = \arg\min_{x\in \R^d} \norm{Ax - z}_p$. If $x\in\R^d$ satisfies $\norm{Ax-z}_p \leq (1+c\gamma) R$, we have that $\norm{Ax^*-Ax}_p \leq \sqrt{\gamma} R$, where $c\in (0,1]$ is an absolute constant.
\end{lemma}

\begin{proof}
The same statement is proved for matrices with bounded Lewis weights in \cite[Theorem 3.19]{MMWY}. Now we prove for a general matrix $A \in \R^{n \times d}$. Let $w_i(A)$ be the Lewis weight of row $a_i$ and $k_i = \lceil\frac{w_i(A)}{d/n}\rceil$. We replace each row $a_i$ by $k_i$ copies of $\frac{a_i}{k_i^{1/p}}$, obtaining a new matrix $A'$. Then we have
\[
    \frac{a_i^{\top}}{k_i^{1/p}} \left(\sum_{i=1}   ^{n} k_i \cdot (\frac{w_i}{k_i})^{1-\frac{2}{p}} \frac{a_i}{k_i^{1/p}} \frac{a_i^{\top}}{k_i^{1/p}} \right)^{\frac{p}{2}} \frac{a_i}{k_i^{1/p}} = \frac{a_i^{\top}}{k_i^{1/p}} \left(\sum_{i=1}^{n} w_i^{1-\frac{2}{p}} a_i a_i^{\top}\right)^{\frac{p}{2}} \frac{a_i}{k_i^{1/p}} = \frac{w_i}{k_i}.
\]
Therefore, the Lewis weight $w_i(A')$ is bounded by $\frac{d}{n}$. It is also clear that $\norm{Ax}_p = \norm{A'x}_p$. We also split every entry of $z$ into $k_i$ copies, obtaining $z'$. Thus we have $\norm{Ax - z}_p = \norm{A'x - z'}_p$. Note that $\arg\min_{x \in \R^d} \norm{Ax - z}_p = \arg\min_{x \in \R^d} \norm{A'x - z'}_p$. Hence we can use~\cite[Theorem 3.19]{MMWY} to get $\norm{Ax^* - Ax}_p = \norm{A'x^* - A'x}_p \leq \cO(\sqrt{\gamma} R)$. This completes the proof.
\end{proof}

\begin{theorem}\label{thm:p-before-bounding}
Let $A \in \R^{n\times d}$, $b\in\R^n$. Suppose that $S$ is a rescaled sampling matrix according to $w_i^{\OL}(A)$ with oversampling factor $\beta = \Theta(\frac{\log^2 d}{\eps} \log n \log \frac{1}{\delta})$ and $S'$ is a rescale dsampling matrix according to $w_i^{\OL}(A)$ with oversampling factor $\beta' = \Theta(\log d)$. Let $x_c = \arg \min_{x\in \R^d} \norm{S'Ax - S'b}_p$, $z = b -Ax_c$, $\hat{x} = \arg \min_{x\in \R^d} \norm{SA x - Sz}_p$ and $\tilde{x} = x_c + \hat{x}$. It holds that
\[
   \norm{A\tilde{x}-z}_p \leq (1+\eps) \min_{x\in\R^d} \norm{Ax-z}_p
\]
with probability at least $0.99-\delta$ and the query complexity is
\[
\cO\left(\frac{d}{\eps} \log^2 d \log^2 n \log \kappa^{\OL}(A) \log \frac{1}{\delta} \right)
\]
for $1<p<2$ and
\[
\cO\left(\frac{d}{\eps} \log^2 d \log^2 n \log \frac{\norm{A}_2}{\sigma} \log \frac{1}{\delta} \right)
\]
for $p=2$.
\end{theorem}

\begin{proof}
Let $c\in (0,1]$ be the same absolute constant in Lemma~\ref{lem:gamma-correctness} and 
\begin{equation*}
K = 
    \begin{cases}
        \Theta(\frac{d}{c^2} \log^2 d \log^2 n \log \kappa^{\OL}(A) \log \frac{1}{\delta}), \quad 1<p<2\\
        \Theta(\frac{d}{c^2} \log^2 d \log^2 n \log \frac{\norm{A}_2}{\sigma} \log \frac{1}{\delta}), \quad p=2.
    \end{cases}
\end{equation*}

When the query complexity is $K\eps$, it follows from Lemma~\ref{lem:Ax-z-optimal-approximation} that $\norm{A\tilde{x} - b}_p \leq (1 + c\sqrt{\eps}) R$ with probability at least $0.99 - \delta$. Let $x^* = \arg \min_{x\in\R^d} \norm{Ax - z}_p$ and $\gamma = c\sqrt{\eps}$, then $\norm{A\bar{x} - Ax^*}_p \leq \sqrt{\gamma} R$ by Lemma~\ref{lem:gamma-correctness}. 

Let $\eps_1$ be such that $K/\eps = \gamma K/\eps_1^2$, thus $\eps_1 = c^{1/2}\eps^{3/4}$.
By Lemma~\ref{lem:gamma-assumption-eps-approx}, we have $\norm{A \tilde{x} - b}_p \leq (1+c\eps_1) R$ with probability at least $1-\cO(\delta)$. Note that here the failure probability is not $0.01+\cO(\delta)$ because we have already assumed that $\norm{Sz}_p^p = \cO(\norm{z}_p^p)$ with probability at least $0.995$ 
and $\norm{z}_p^p \leq R^p$ with probability at least $0.995$.

Using Lemma~\ref{lem:gamma-correctness}, we can get that $\norm{A\tilde{x} -Ax^*}_p \leq \sqrt{\eps_1}R$. Now, taking $\gamma = \eps_1$ in Lemma~\ref{lem:gamma-eps-guarantee}, we see that,  with probability at least $1-\cO(\delta)$,
$\norm{A \tilde{x} - b}_p \leq (1+\eps_2) R$ for $\eps_2$ such that $\eps_1 K/\eps_2^2 = K/\eps$, i.e, $\eps_2 = c^{1/4}\eps^{7/8}$. Repeating this process, we can obtain that 
$\norm{A \tilde{x} - b}_p \leq (1+\eps_i) R$ with probability at least $1-\cO(i\cdot\delta)$, where $\eps_i^2 = \eps_{i-1}\eps$. We can solve that $\eps_i = c^{\frac{1}{2^i}}\eps^{1-\frac{1}{2^{i+1}}}$. Letting $i = \log \log (1/\eps)$ yields $\eps_i \leq 2\eps$, that is, $\norm{A \tilde{x} - b}_p \leq (1+2\eps) R$ with probability at least $1-\cO(\delta\log\log(1/\eps))$. Rescaling $\eps = \eps/2$ and $\delta = \Theta(\delta/\log \log(1/\eps))$ completes the proof.
\end{proof}

\begin{theorem}[Main results]
\label{thm:general-p-opt}
Let $A \in \R^{n\times d}$ and $b \in \R^{n}$. 

Algorithm~\ref{alg:p-simple-OAR} modified as in Theorem~\ref{thm:p-compression-OAR-boost} outputs a solution $\tilde{x}$ which satisfies that
\begin{equation}
\Norm{A\tilde{x} - b}_p \leq (1+\eps) \min_{x\in\R^d}  \Norm{Ax-b}_p
\end{equation}
with probability at least $1-\delta$ and makes
\[
\cO \left(\frac{d}{\eps}\log^2 d \log^2 \frac{d}{\delta\eps} \cdot \log \frac{n\kappa^{\OL}(A)}{\delta} \log^2 \frac{1}{\delta} \right)
\]
queries overall in total.

Furthermore, with probability at least $1-\delta$, it uses $\cO(md)$ words of space in total.

Algorithm~\ref{alg:online-active-regression} modified by the sampling scheme in Theorem~\ref{thm:p-compression-OAR-boost} makes 
\[
m = \cO\left(\frac{d}{\eps} \log^2 d \log^2 \frac{d}{\delta\eps} \cdot \log\left(n \frac{\norm{A}_2}{\sigma\delta}\right) \log^2 \frac{1}{\delta} \right)
\]
queries in total and maintains for each $T=d+1,\dots,n$ a solution $\tilde{x}^{(T)}$ which satisfies that
\[
\Norm{A^{(T)}\tilde{x}^{(T)} - b^{(T)}}_2 \leq (1+\eps) \min_{x\in\R^d} \Norm{A^{(T)} x - b^{(T)}}_2.
\]

Furthermore, with probability at least $1-\delta$, it uses $\cO(md)$ words of space in total.
\end{theorem}

\begin{proof}
we shall only prove for the case $1<p<2$ below, as the case $p=2$  follows the same approach with a different sum of online Lewis weights. Using the boosting procedure explained prior to  Theorem~\ref{thm:p-compression-OAR-boost}, we can obtain a constant factor approximation with probability at least $1-\delta$. 

The proof is similar to that of Theorem~\ref{thm:p-compression-OAR}.
We can write $\tilde{A}_1 = S_1 A$ for a sampling matrix $S_1$ with respect to the online Lewis weights.
It follows from Theorem~\ref{thm:p-before-bounding} that with probability at least $0.995-\delta$,
\[
\norm{S_1 A \hat{x} - S_1 z}_p^p \leq (1 + \eps)\min_{x\in\R^d} \norm{S_1 A x - S_1 z}_p^p,
\]
where $z = b - A x_c$. 
Following the same approach in the proof of Theorem~\ref{thm:p-compression-OAR}, we can obtain that,
with probability at least $0.995 - \cO(\delta)$,
\[
    \norm{A \tilde{x} - b}_p^p \leq \left(1 + \cO(\eps)\right)\min_{x\in\R^d} \norm{A x - b}_p^p.
\]

For the results above to go through, $S_1$ should have oversampling parameter $\beta_1 = \Theta(\frac{d}{\delta^{2+p} \eps^{2+p}} \log\frac{1}{\eps\delta})$, resulting in 
\[
    N = \cO\left(\beta_1 \sum_{i=1}^n w_i^{\OL}(A)\right) = \cO\left(\frac{d^2}{\delta^{2+p} \eps^{2+p}} \log\frac{1}{\eps\delta} \log n\log \kappa^{\OL}(A) \right)
\]
rows of $S_1A$ with probability at least $1-\delta$. The $\frac{1}{\delta^{2+p}}$ term in $\beta_1$ is for bounding $\norm{S_1 z}_p^p = \cO(\norm{z}_p^p)$ with probability at least $1-\delta$. Also, $S_3$ should have an oversampling parameter 
\[
    \beta_3 = \Theta\left( \frac{\log^2 d}{\eps}\log\frac{1}{\delta}\log N \right) = \cO\left( \frac{\log^2 d}{\eps} \log\frac{d}{\eps \delta} \log\frac{1}{\delta} \right),
\]
resulting in
\[
    m = \cO\left(\beta_3 \sum_{i=1}^n w_i^{\OL}(SA)\right) 
      = \cO\left(\frac{d\log^2 d}{\eps} \log^2 \frac{d}{\eps \delta} \log \frac{n\kappa^{\OL}(A)}{\delta} \log\frac{1}{\delta} \right) 
\]
rows of $S_3 S_1 A$ with probability at least $1-\delta$. Here we upper bound $\kappa^{\OL}(SA)$ by Lemma~\ref{lem:sum-sub-online-LW}. 
Now the constant failure probability $0.005$ is from bounding $\norm{S_3(S_1z - S_1A \hat{x}_c)}_p^p = \cO(\norm{z - S_1 A \hat{x}_c}_p^p)$. In order to obtain $1-\delta$ success probability, we use the same boosting method for Theorem~\ref{thm:p-compression-OAR-boost} and sample $\log \frac{1}{\delta}$ independent copies of $S_3$, which results in 
\[
m = \cO\left(\frac{d\log^2 d}{\eps} \log^2 \frac{d}{\eps \delta} \log \frac{n\kappa^{\OL}(A)}{\delta} \log^2 \frac{1}{\delta} \right).
\]

The total number of queried labels is dominated by $m$. Rescaling $\eps$ and $\delta$ gives the claimed result.
\end{proof}

\section{Experiments}
In this section, we provide empirical results on online active $\ell_p$ regression with $p=1$, $p=1.5$ and $p=2$. We compare our methods with online uniform sampling, the offline active regression algorithms~\cite{MMWY,CD21,PPP21} and the thresholding algorithm in \cite{RJZ17}. The quantity we compare is the relative error, defined as $(err - err_{\opt})/err_{\opt}$, where $err = \|A\tilde{x}-b\|_p$ is the error of the algorithm's output $\tilde{x}$ and $err_{\opt} = \min_x\|Ax-b\|_p$ is the minimum error of the $\ell_p$ regression. 
All algorithms are prescribed with a budget for querying the labels and we explain the online uniform sampling algorithm, thresholding algorithm and the adaptation of online and offline active regression algorithms to the budget-constrained setting below.
\begin{itemize}
    \item Online Uniform Sampling: In the $t$-th round, we sample the new data point $[a_t,b_t]$ with probability $B_t/(n-t)$, where $B_t$ is the remaining budget. 
    \item Regression via Thresholding: We use Algorithm 1.b in \cite{RJZ17}, which is for $\ell_2$ regression only, and assign the weights $\xi_{i}=1$ for all $i \in [n]$. 
    \item Online Active Regression: We sample each data point with probability proportional to $\tilde{w}_t$, where $\tilde{w}_t$ is the approximate online Lewis weight calculated with the compression technique for $p=1$ and $p=1.5$. 
    \item Offline Active Regression: For $p=1$, the algorithms in~\cite{CD21,PPP21} are under the budget setting and no modification is needed. For $p=1.5$ and $p=2$, we use the offline algorithm in~\cite{MMWY} which has the optimized dependence on $\eps$. However, the algorithm is recursive so we can not control the exact budget we use. We treat the error with budget in $[x, x+100]$ as the error with budget $x+100$. For both $p=1.5$ and $p=2$, the offline algorithm~\cite{MMWY} involves parallel sampling. Since it expects to sample $\cO(d)$ data points for a constant-factor approximation, we allocate a budget of size $d$ to the constant-factor approximation and allocate the remaining budget to the regression on residuals.
\end{itemize}

We perform experiments on both synthetic and real-world data sets to demonstrate the efficacy of our approaches.
\begin{itemize}
    \item 
Synthetic Data: We generate the synthetic data as follows. Each row of $A\in\R^{n\times d}$ is a random Gaussian vector, i.e., $a_i\sim\mathcal{N}(0,I_d)$. The label is generated as $b=Ax^*+\xi$ where $x^*$ is the ground truth vector and  $\xi$ is the  Gaussian noise vector, i.e., $\xi\sim \mathcal{N}(0,1)$. To make the rows of $A$ have nonuniform Lewis weights, we enlarge $d$ data points by a factor of $n^{\frac{1}{p}}$. We choose $n=10000$ and $d=100$.
    \item Real-world Data: We evaluate our algorithm on a real-world dataset, the gas sensor data~\cite{vergara2012chemical,rodriguez2014calibration} from the UCI Machine Learning Repository\footnote{\url{https://archive.ics.uci.edu/ml/datasets/Gas+Sensor+Array+Drift+Dataset+at+Different+Concentrations}}. The dataset contains 13910 measurements of chemical gas characterized by 128 features and their concentration levels.
\end{itemize}

We vary the budget sizes between 800 and 1400 (8\%--14\% of the data size) for the synthetic data and between $1600$ and $2500$ (12\%--18\% of the data size) for the real-world data. We present our results in budget-versus-error plots in Figures~\ref{fig:l1}--\ref{fig:l2}. All experiments are repeated 20 times. The mean relative errors and standard deviations are reported in the plots. All our experiments are conducted in MATLAB on a Macbook Pro with an i5 2.9Hz CPU and 8GB of memory. 

\subsection{Experiment Results}
Below we present the experiment results for the online active $\ell_p$ regression, $p=1,1.5,2$. The results are plotted in Figures~\ref{fig:l1}, \ref{fig:l1.5}, \ref{fig:l2}, respectively.

\begin{figure}[t]
\centering 
\subfloat[Synthetic data]{%
  \includegraphics[clip,trim=0cm 0.3cm 0cm 0cm, scale=0.9]{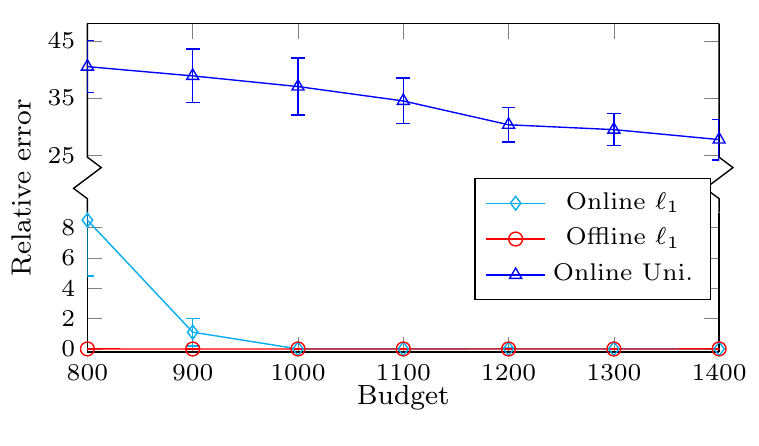}%
}
\subfloat[Gas sensor data]{%
   \includegraphics[clip,trim=0cm 0.3cm 0cm 0cm, scale=0.9]{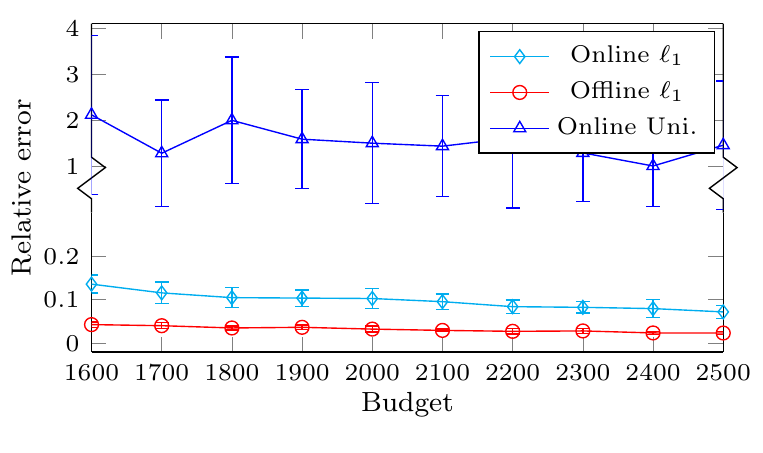}
 }
\caption{Performance of algorithms for online $\ell_1$ active regression on both synthetic data and Gas sensor data.}\label{fig:l1}
\end{figure}

\begin{figure}[t]
    \centering
\subfloat[Synthetic data]{
\includegraphics[clip,trim=0cm 0.3cm 0cm 0cm,scale=0.9]{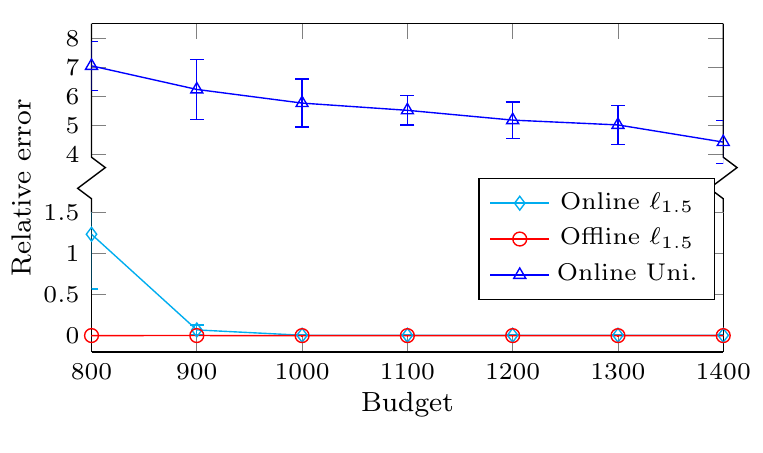}
}
\subfloat[Gas Sensor data]{
\includegraphics[clip,trim=0cm 0.3cm 0cm 0cm,scale=0.9]{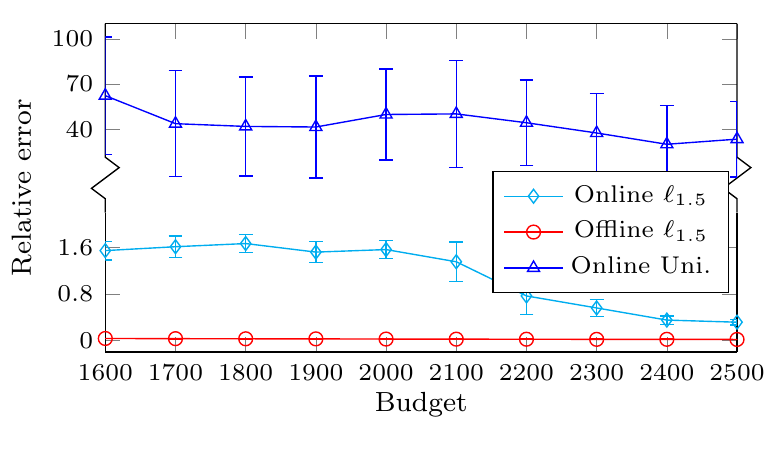}
}
\caption{Performance of algorithms for online $\ell_{1.5}$ active regression on both synthetic data and Gas Sensor data.}\label{fig:l1.5}
\end{figure}

\begin{figure}[t]
\centering 
\subfloat[Synthetic data]{%
    \includegraphics[clip, trim=0 0.3cm 0 0,scale=0.9]{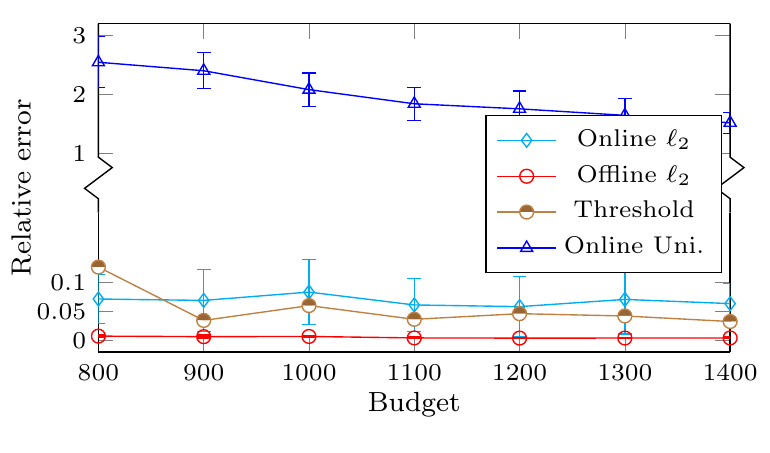}
}
\subfloat[Gas Sensor data]{%
   \includegraphics[clip, trim=0 0.3cm 0 0,scale=0.9]{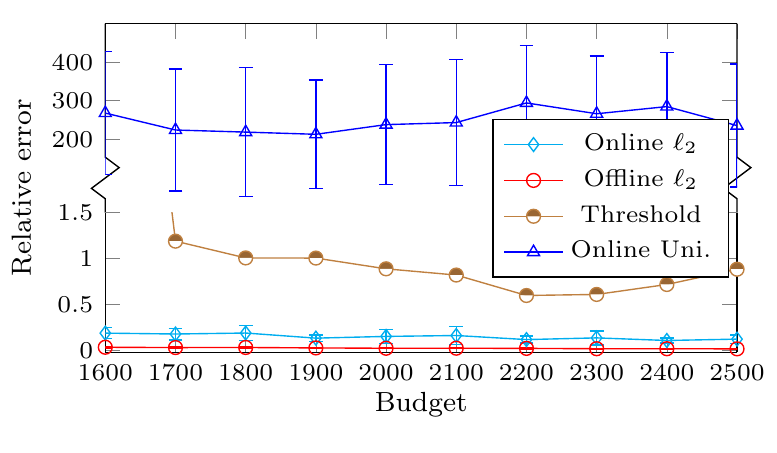}
}
    \caption{Performance of algorithms for online $\ell_2$ active regression on both synthetic data and Gas Sensor data.}\label{fig:l2}
\end{figure}

\begin{itemize}
\item $p=1$:
For the synthetic data, we see that the online regression algorithm achieves a relative error comparable to that of the offline regression algorithm when the budget is at least $1000$ and always significantly outperforms the online uniform sampling algorithm. 
For the real-world data, the online regression algorithm's performance is again significantly better than the online uniform sampling algorithm and comparable to that of the offline active regression algorithm.

\item $p=1.5$:
The online $\ell_{1.5}$ regression algorithm significantly outperforms the online uniform sampling on both data sets. It achieves a relative error comparable to that of the offline active regression algorithm on the real-world data and is only slightly worse than the offline algorithm when the budget size is at least $2300$ (14.3\% of the data size).

\item $p=2$:
The online $\ell_2$ regression algorithm significantly outperforms the online uniform sampling on both data sets and performs much better than the thresholding algorithm on real-world data. It achieves a relative error comparable to that of the offline active regression algorithm on the synthetic data and is only slightly worse than the offline algorithm on real-world data.
\end{itemize}

\section{Conclusion}
We provably show an online active regression algorithm which uses sublinear space for the $\ell_p$-norm, $p\in [1,2]$. Our experiments demonstrate the superiority of the algorithm over online uniform sampling on both synthetic and real-world data and comparable performance with the offline active regression algorithm.

\newpage
\appendix

\section{Some Facts of Lewis Weights}

We assume that $p\in [1,2]$ in this section, thus $1-2/p \leq 0$. Also we assume that $A\in \R^{n\times d}$ has full column rank in the statements of the lemmata and theorems, which is not a true restriction because Lewis weights are invariant under invertible linear transformation of $A$, i.e., $w_i(A) = w_i(AT)$ for any invertible $T$. When $A$ does not have full column rank, we can find $T$ such that $AT = \begin{pmatrix}A' & 0\end{pmatrix}$, where $A'$ has full column rank and $0$ is a zero matrix. Then $(A^\top W^{1-2/p}A)^\dagger = T\begin{pmatrix} ((A')^\top W^{1-2/p}A')^{-1} & \\ & 0\end{pmatrix} (T^\top)^{-1}$ and all the proofs in this section will go through by acting on $A'$ and $((A')^\top W^{1-2/p}A')^{-1}$.

\begin{lemma}\label{lem:core_matrix_approx}
Given $A\in\R^{n\times d}$ with $\ell_p$ Lewis weights $w_i,i\in[n]$, let $S$ be the rescaled sampling matrix with respect to $p_1,\dots,p_n$ satisfying that $\min\{\beta w_i,1\} \leq p_i \leq 1$, where $\beta = \Omega(\eps^{-2}\log (d/\delta))$. With probability at least $1-\delta$, it holds that
\[
(1-\eps)\sum_{i=1}^{n} w_i^{1-\frac{2}{p}} a_ia_i^\top \preceq \sum_{i=1}^{n}\frac{(\mathbbm{1}_S)_i}{p_i} w_i^{1-\frac{2}{p}} a_ia_i^\top\preceq(1+\eps) \sum_{i=1}^{n} w_i^{1-\frac{2}{p}} a_ia_i^\top.
\]
\end{lemma}
\begin{proof}
We prove the lemma by a matrix Chernoff bound. Without loss of generality, we assume that $p_i \leq 1/\beta$ for all $i$, otherwise we can restrict the sum to the $i$'s such that $p_i\leq 1/\beta$. We further assume that $A^\top W^{1-\frac{2}{p}} A = I_d$, where $W = \diag\{w_1,\dots,w_n\}$. Let $X_i=\frac{(\mathbbm{1}_S)_i}{p_i}\cdot w_i^{1-\frac{2}{p}} a_ia_i^\top - w_i^{1-\frac{2}{p}} a_ia_i^\top$, then $\E{X_i}=0$. By the definition of Lewis weights, we have $w_i^{\frac{2}{p}} = a_i^\top (A^\top W^{1-\frac{2}{p}} A)^{-1} a_i$. Hence, we have $\norm{a_i}_2^2 = w_i^\frac{2}{p}$. Next, $\norm{X_i}_2 \leq \frac{w_{i}^{1-\frac{2}{p}}}{p_i} \norm{a_i}_2^2 = \frac{1}{\beta} w_{i}^{-\frac{2}{p}} \norm{a_i}_2^2 =\frac{1}{\beta}$ and
\[\Norm{\E \left( \sum_{i=1}^{n} X_iX_i^\top \right)}_2 = \Norm{\E \left(\sum_{i=1}^{n} (1-\frac{1}{p_i}) w_i^{2(1-\frac{2}{p})} \Norm{a_i}_2^2 a_ia_i^\top \right)}_2 \leq \Norm{\sum_{i=1}^{n} w_i^{1-\frac{2}{p}} \cdot \frac{a_ia_i^\top}{\beta}}_2
= \frac{1}{\beta}.\] 
Applying the matrix Chernoff inequality, we have
\[
\Pr\left\{\Norm{\sum_{i=1}^{n} X_i X_i^\top}_2 \geq \eps\right\} \leq 2d\exp{\left(\frac{-\eps^2}{\frac{1}{\beta} + \frac{\eps}{3\beta}} \right)} 
= 2d\exp{\left(-\Omega(\beta\eps^2)\right)}\\
\leq \delta. 
\]
\end{proof}

\begin{lemma}\label{lem:good_Lewis_initial}
Suppose that $A\in \R^{n\times d}$ and $\overline{w_1},\dots,\overline{w_n}$ are the Lewis weights of $A$. Let $w_1,\dots,w_n$ be weights such that
\[
\alpha w_i^{2/p} \leq a_i^\top \left(\sum_i w_i^{1-2/p} a_i a_i^\top\right)^{-1} a_i \leq \beta w_i^{2/p},\quad \forall i=1,\dots,n,
\]
then $\alpha w_i \leq \overline{w_i} \leq \beta w_i$ for all $i$.
\end{lemma}
\begin{proof}
Let $\gamma = \sup \{c > 0 : w_i\geq c \overline{w_i}\text{ for all }i\}$. It then holds for all $i$ that
\begin{align*}
w_i^{2/p} &\geq \frac{1}{\beta} a_i^\top \left(\sum_i w_i^{1-2/p} a_i a_i^\top\right)^{-1} a_i  \\
&\geq \frac{1}{\beta} a_i^\top \left(\sum_i (\gamma \overline{w_i})^{1-2/p} a_i a_i^\top\right)^{-1} a_i \\
&= \frac{\gamma^{2/p-1}}{\beta}a_i^\top \left(\sum_i \overline{w_i}^{1-2/p} a_i a_i^\top\right)^{-1} a_i \\
&= \frac{\gamma^{2/p-1}}{\beta} \overline{w_i}^{2/p}.
\end{align*}
This implies that 
\[
\gamma^{2/p} \geq \frac{\gamma^{2/p-1}}{\beta},
\]
and thus
\[
\gamma \geq \frac{1}{\beta},
\]
that is, $w_i\geq \overline{w_i}/\beta$ for all $i$. Similarly one can show that $w_i\leq \overline{w_i}/\alpha$.
\end{proof}

Combining Lemmas \ref{lem:core_matrix_approx} and \ref{lem:good_Lewis_initial} leads to the following lemma.
\begin{lemma}\label{lem:compression_lewis_weights_appendix}
Let $A_i\in \R^{n_i\times d}$ $(i=1,\dots,r)$, $B\in \R^{k\times d}$ and $M = A_1 \circ A_2 \circ \cdots \circ A_r \circ B$. For each $i\in [r]$, let $S_i\in \R^{m_i\times n_i}$ be the rescaled sampling matrix with respect to $p_{i,1},\dots,p_{i,n_i}$ with $\min\{\beta w_j(A_i),1\} \leq p_{i,j}\leq 1$ for each $j\in [n_i]$, where $\beta = \Omega(\eps^{-2}\log (d/\delta))$. Let $M' = S_1 A_1 \circ \cdots \circ S_r A_r \circ B$. The following statements hold with probability at least $1-\delta$.
\begin{enumerate}
\item For each $i\in [r]$ and each $j \in [m_i]$, it holds that
\[
(1-\eps)\frac{w_{n_1+\cdots+n_{i-1}+s_i(j)}(M)}{p_{i,s_i(j)}} \leq w_{m_1+\cdots+m_{i-1}+j}(M') \leq (1+\eps) \frac{w_{n_1+\cdots+n_{i-1}+s_i(j)}(M)}{p_{i,s_i(j)}},
\]
where $s_i(j)\in [n_i]$ is the row index such that $(S_i)_{j,s(i,j)}\neq 0$.
\item 
For each $j=1,\dots,k$, it holds that
\[
(1-\eps)w_{n_1+\cdots+n_r+j}(M) \leq w_{m_1+\cdots+m_r+j}(M') \leq (1+\eps) w_{n_1+\cdots+n_r+j}(M).
\]
\end{enumerate}
\end{lemma}
\begin{proof}
Define partial sums $\mu_i = m_1 + \cdots + m_i$ with $\mu_0 = 0$ and $\nu_i = n_1 + \cdots + n_i$ with $\nu_0 = 0$. For each $j\in [\mu_r + k]$,
\[
w'_j = \begin{cases}
        w_{\nu_{i-1}+s_i(j)}(M)/p_{i,s_i(j)}, & \mu_{i-1} < j\leq \mu_i; \\
        w_{j - \mu_r + \nu_r}(M), & j\geq \mu_r.
      \end{cases}
\]
and
\[
L = \sum_{i=1}^r\sum_{j=1}^{m_i} \frac{(S_iA_i)_j (S_iA_i)_j^\top}{(w'_{\mu_{i-1} + j})^{p/2-1}} + \sum_{j=1}^{k} \frac{b_i b_i^\top}{(w'_{\mu_r + j})^{p/2-1}}.
\]
Then we have
\[
L = \sum_{i=1}^{r}\sum_{j=1}^{m_i} \frac{(A_{i})_{s_i(j)} (A_{i})_{s_i(j)}^\top}{p_{i,s_i(j)} (w_{s_i(j)}(A_i))^{p/2-1}} + \sum_{j=1}^{k} \frac{b_i b_i^\top}{(w_{\nu_r + j}(M))^{p/2-1}}.
\]
Let $W_M = \diag\{w_1(M),\dots,w_{\nu_r + k}(M)\}$. Let $p_i = 1$ for $i = \nu_r + 1,\dots,\nu_r + k$. Also note that $p_{i,j} \geq \min\{\beta w_{\nu_{i-1}+j}(M),1\}$ since $w_j(A_i)\geq w_{\nu_{i-1}+j}(M)$. 
It follows from Lemma~\ref{lem:core_matrix_approx} that
\[
(1-\eps)(M^\top W_M^{1-2/p} M)\preceq  L \preceq (1+\eps)(M^\top W_M^{1-2/p} M),
\]
with probability at least $1-\delta$.

Next we verify that $\{w_j'\}_j$ are good weights for $M'$. When $\mu_{i-1} < j\leq \mu_i$,
\begin{multline*}
(w'_j)^{2/p} = \frac{ (w_{\nu_{i-1}+s_i(j)}(M))^{2/p} }{ p_{i,s_i(j)}^{2/p} } = \frac{(A_i)_{s_i(j)} (M^\top W_M^{1-2/p} M)^{-1} (A_i)_{s_i(j)}^\top}{p_{i,s(j)}^{2/p}} \\
= \frac{1}{1\pm \eps}\cdot\frac{(A_i)_{s_i(j)} L^{-1} (A_i)_{s_i(j)}^\top}{p_{i,s_i(j)}^{2/p}} = \frac{1}{1\pm \eps} (S_iA_i)_j L^{-1} (S_iA_i)_j^\top,
\end{multline*}
where $(S_iA_i)_j$ denotes the $j$-th row of $S_iA_i$. Similarly, one can show that for $j > \mu_r$, 
\[
(w'_{\mu_r + j})^{2/p} = \frac{1}{1\pm\eps} b_{j - \mu_r} L^{-1} b_{j - \mu_r}^\top.
\]
The result follows from Lemma~\ref{lem:good_Lewis_initial}.
\end{proof}

\subsection{Online \texorpdfstring{$\ell_p$}{lp} Lewis Weights}\label{sec:online_lewis_weights}
The goal of this section is to show Lemma~\ref{lem:sum_of_online_lewis_weights}, which states that the sum of the online $\ell_p$ Lewis weights of a matrix $A\in\R^{n\times d}$ is upper bounded by $O(d\log n \log \kappa^{\OL}(A))$ for $p\in [1,2)$. This is a generalization of \cite[Lemma 5.15]{sliding_window} from $p=1$ and we follow the same approach in~\cite{sliding_window}.

\begin{lemma} \label{lem:uniform}
If the leverage scores of $A$ are at most $C>0$, then the $\ell_p$ Lewis weights of $A$ are at most $C$ for $p\in[1,2]$.
\end{lemma}
\begin{proof}
This is the generalization of Lemma 5.12 in~\cite{sliding_window} and we follow the same proof approach.

By the assumption, we have $a_i^\top(A^\top A)^{-1}a_i\leq C$ for $i\in[n]$. We prove by induction that for iteration $j$ in the Lewis weight iteration, we have $W^{(j)}\preceq C^{1-(1-p/2)^j}I_n$.

For the base case $j=1$, we have $W_{i,i}^{(j)}=(a_i^\top(A^\top A)^{-1}a_i)^{p/2}\leq C^{p/2}$. Thus $W^{(1)}\preceq C^{p/2}I_n$ as desired. 

For iteration $j$, by the induction hypothesis, we have $W^{(j-1)}\preceq C^{1-(1-p/2)^{j-1}}I_n$, which implies that $(W^{(j-1)})^{1-2/p}\succeq C^{(1-(1-p/2)^{j-1})(1-2/p)}I_n$ since $1 - 2/p \leq 0$. Thus,
\[ 
    A^\top(W^{(j-1)})^{1-2/p}A \succeq C^{(1-(1-p/2)^{j-1})(1-2/p)}A^\top A, 
\]
and
\[ 
    (A^\top(W^{(j-1)})^{1-2/p}A)^{-1}\preceq C^{(1-(1-p/2)^{j-1})(2/p-1)}(A^\top A)^{-1}. 
\]
It then follows from~\eqref{eq:lw_iter} that
\begin{align*}
    (W_{i,i}^{(j)})^{2/p}= a_i^\top(A^\top (W^{(j-1)})^{1-2/p}A)^{-1}a_i &\leq C^{(1-(1-p/2)^{j-1})(2/p-1)}a_i^\top(A^\top A)^{-1}a_i\\
    &\leq C^{(1-(1-p/2)^{j-1})(2/p-1)+1}.
\end{align*}
Notice that $((1-(1-p/2)^{j-1})(2/p-1)+1)p/2=1-(1-p/2)^j$, we have obtained that $W_{i,i}^{(j)}\leq C^{1-(1-p/2)^j}$ for all $i$, i.e., $W^{(j)}\preceq C^{1-(1-p/2)^j}I_n$. The induction step is established.

The claim follows the convergence of Lewis weight iteration~\cite{CP15}.
\end{proof}

\begin{lemma} \label{lem:split}
Given $A=[a_1, \dots, a_n]^\top\in\R^{n\times d}$, let $B\in\R^{(n+1)\times d}=[a_1, \dots, a_{j-1}, b_j, a_{j+1}, \dots, a_n, b_{n+1}]^\top$ where $b_j = (1-\gamma)^{1/p}a_j$ and $b_{n+1}=\gamma^{1/p}a_j$ for some $\gamma\in[0,1]$ and $j\in[n]$. Then we have $w_i(A)=w_i(B)$ for $i\neq j, n+1$, $w_j(B)=(1-\gamma)w_j(A)$ and $w_{n+1}(B)=\gamma w_j(A)$.
\end{lemma}

\begin{proof}
Without loss of generality, we suppose $j=n$. Let $W\in\R^{n\times n}$ be the diagonal Lewis weight matrix of $A$, i.e., $W_{i,i}=w_i(A)$. Let $\overline{W}^{(n+1)\times (n+1)}$ be a diagonal matrix where $\overline{W}_{i,i}=w_i(A)$ for $i=1,\dots,n-1$, $\overline{W}_{n,n}=(1-\gamma)w_n(A)$ and $\overline{W}_{n+1,n+1}=\gamma w_n(A)$. According to the uniqueness of Lewis weights, it suffices to show that $\tau_i(\overline{W}^{1/2-1/p}B)=\overline{W}_{i,i}$ for $i\in[n+1]$.

Notice that the first $n-1$ rows of $\overline{W}^{1/2-1/p}B$ are the same as those of $\overline{W}^{1/2-1/p}A$. The last two rows of $\overline{W}^{1/2-1/p}B$ are $w_n(A)^{1/2-1/p}(1-\gamma)^{1/2-1/p}(1-\gamma)^{1/p}a_n=w_n(A)^{1/2-1/p}(1-\gamma)^{1/2}a_n$ and $w_n(A)^{1/2-1/p}\gamma^{1/2}a_n$, respectively. Thus we have $\|W^{1/2-1/p}Ay\|^2_2=\|\overline{W}^{1/2-1/p}By\|^2_2$ for any vector $y$, which indicates that the leverage scores of the first $n-1$ rows of $W^{1/2-1/p}A$ are the same as those of $\overline{W}^{1/2-1/p}B$, i.e., $\tau_i(\overline{W}^{1/2-1/p}B)=W_{i,i}=\overline{W}_{i,i}$ for $1\leq i\leq n-1$.

For the last two rows, we have $\tau_n(\overline{W}^{1/2-1/p}B)=(1-\gamma)\tau_n(W^{1/2-1/p}A)=\overline{W}_{n,n}$ and $\tau_{n+1}(\overline{W}^{1/2-1/p}B)=\gamma\cdot\tau_n(W^{1/2-1/p}A)=\overline{W}_{n+1,n+1}$. Thus we have  $\tau_i(\overline{W}^{1/2-1/p}B)=\overline{W}_{i,i}$ for all $i\in[n+1]$.
\end{proof}

\begin{corollary}
For any matrix $A\in\R^{n\times d}$. Let $B\in\R^{n\times d}$ have the same rows but with the $j$-th row reweighted by a factor $\alpha\in[0,1]$. Then for all $i\neq j$, $w_i(B)\geq w_i(A)$.
\end{corollary}
\begin{proof}
Let $\gamma=1-\alpha^p$ and $\bar{B}\in\R^{(n+1)\times d}=[a_1, \dots, a_{j-1}, (1-\gamma)^{1/p}a_j, a_{j+1}, \dots, a_n, \gamma^{1/p} a_j]^\top$. By Lemma \ref{lem:split}, we have $w_i(\bar{B})=w_i(A)$ for $i\neq j$. Then by Lemma \ref{lem:monotonicity} we have $w_i(B)\geq w_i(\bar{B})=w_i(A)$.
\end{proof}

As mentioned at the beginning of the section, we follow the approach in~\cite{sliding_window} to upper bound the sum of the online Lewis weights. It makes a critical use of an upper bound on the sum of online $\lambda$-leverage scores of a matrix but does not provide a proof for the case of small $\lambda$.  We reproduce the upper bound and provide a proof for completeness below. 

\begin{definition}[Online Ridge Leverage Scores]\label{def:ridge}
Let $A \in \R^{n\times d}$ and $\lambda \geq 0$ be a regularization parameter. The online ridge leverage score of row $a_i\in \R^{d}$ is defined to be $\tau^{\OL}_{i}(A;\lambda) = \min \{ a_i^{\top} ((A^{(i-1)})^{\top}A^{(i-1)} + \lambda I_d)^{-1} a_i, 1 \} $.
\end{definition}

The next lemma upper bounds the sum of online ridge leverage scores, which was stated as {\cite[Lemma 2.2]{sliding_window}}.
\begin{lemma}[Bound on Sum of Online Ridge Leverage Scores]
\label{lem:sum-online-ridge-leverage-score}
Suppose that $A\in \R^{n\times d}$ and $\sigma^\ast = \min_{i\in [n]} \sigma_{\min}(A^{(i)})$, where $\sigma_{\min}$ denotes the smallest singular value of a matrix. Let $\lambda > 0$. It holds that
\[
\sum_{i=1}^n \tau_i^{\OL}(A;\lambda) = \begin{cases}
    \cO(d\log(\norm{A}_2^2/\lambda)), & \lambda \geq (\sigma^\ast)^2;\\
    \cO(d\log \kappa^{\OL}(A)), & \lambda < (\sigma^\ast)^2.
\end{cases}
\]
\end{lemma}

The exact case of $\lambda \geq (\sigma^\ast)^2$ was proved in~\cite{CMP}. The claim of the case $\lambda < (\sigma^\ast)^2$ appeared in~\cite{sliding_window} without an explicit proof. We shall supply a proof below for completeness. First we need the following observation of the leverage scores.

\begin{proposition}\label{prop:leverage-score}
For $A\in \R^{n\times d}$, let $\tau_i$ be the leverage score of row $a_i$. Then $\tau_i = \min_{x\in \R^d}  \norm{x}_2^2$ subject to $x^{\top} A = a_{i}^{\top}$. Suppose $\mu\in \R^d$ is the minimizer, then $\mu_{i} \in [0,1]$. 
\end{proposition}
\begin{proof}
Suppose that $A = U \Sigma V^{\top}$ is the singular value decomposition of $A$, where $U\in \R^{n\times d}$ has orthonormal columns and $\Sigma, V\in \R^{d\times d}$ are invertible. Then $\tau_{i} = a_{i}^{\top}(A^{\top} A)^{-1} a_{i} = \norm{e_{i}^{\top} U}_2^2$, where $e_{i}$
is the $i$-th canonical basis vector. Suppose that $x^\top A = a_i^\top = e_i^\top A$. Multiplying both sides by $V\Sigma^{-1}$ leads to $x^\top U = e_i^\top U$ and thus $\norm{x}_2^2 = \norm{x}_2^2 \norm{U}_2^2 \geq \norm{x^\top U}_2^2 = \tau_{i}$.
Next, we show there exists a vector $\mu \in \R^{n}$ satisfying  $\mu^{\top}A = a_{i}^{\top}$ and $\norm{\mu}_2^2 = \tau_{i}$. 

The matrix $U \in \R^{n\times d}$ can be transformed to
$U' = \begin{pmatrix}
I_d\\
0
\end{pmatrix}$ through an orthogonal transformation $E$, that is, $U' = EU$.
Hence, we can find a vector $y\in \R^{n}$ such that $e_{i}^{\top} U = y^{\top} U' = y^{\top} E U$ and $y$ only has nonzero entries in the first $d$ coordinates. Hence, let $\mu = y^{\top} E$ and we get $\mu^\top A = (\mu^\top U)\Sigma V^\top = (e_i^\top U)\Sigma V^\top = a_i^\top$ and $\norm{\mu}_2^2 = \norm{y}_2^2=\norm{e_{i}^{\top} U}_2^2 = \tau_i$.

Without loss of generality, we assume $i=n$. We decompose $\mu$ as $\mu^\top = (\nu^\top \  \mu_{n})$, where $\nu\in\R^{n-1}$ is the first $n-1$ coordinates of $\mu$. Since $\tau_i\in [0, 1]$, we know that $|\mu_{n}| \leq 1$. Next we show that $\mu_n\geq 0$ by contradiction. Suppose that $\mu_n < 0$. Observe that $\mu^{\top} A = (\nu^{\top} \  \mu_{n}) 
\begin{pmatrix}
A^{(n-1)}\\
a_{n}^{\top}
\end{pmatrix}
= \nu^{\top} A^{(n-1)} + \mu_{n} a_{n}^{\top} = a_{n}^{\top}$, whence it follows that $a_n^{\top} = \frac{\nu^{\top} A^{(n-1)}}{1-\mu_{n}}$. Let $x = (\frac{\nu^{\top}}{1-\mu_{n}}\ 0)$, then $x^\top A =  a_{n}^{\top}$ while $\norm{x}_2 = \Norm{\frac{\nu^{\top}}{1-\mu_{n}}}_2 < \norm{\nu^{\top}}_2 \leq \norm{\mu}_2$, contradicting the minimality of $\norm{\mu}_2$.  Therefore, we conclude that $\mu_{n}\geq 0$. 
\end{proof}

Now we are ready to prove Lemma~\ref{lem:sum-online-ridge-leverage-score}.
\begin{proof}\textbf{of Lemma~\ref{lem:sum-online-ridge-leverage-score} }
The case $\lambda \geq (\sigma^{\ast})^2$ is exactly \cite[Theorem 2.2]{CP15}, which actually holds for all $\lambda > 0$. In the remainder of the proof, we assume that $\lambda < (\sigma^{\ast})^2$.

We claim that
\begin{equation}\label{eqn:aux_claim}
\frac{1}{8}\tau_{i+1}^{\OL}(A, \lambda) \leq a_{i+1}^{\top} (A^{(i+1)\top} A^{(i+1)} + \sigma^{*2} I_d )^{-1} a_{i+1}.
\end{equation}
Assuming that the claim is true for now, we will have
\begin{align*}
\sum_{i=1}^{n} \tau_{i}^{\OL}(A, \lambda) &\leq 8\sum_i a_{i}^{\top} (A^{(i)\top} A^{(i)} + \sigma^{*2} I_d )^{-1} a_{i} \\
&\leq 8\sum_i a_{i}^{\top} (A^{(i-1)\top} A^{(i-1)} + \sigma^{*2} I_d )^{-1} a_{i} \\
&= 8\sum_i \tau_{i}^{\OL}(A; \sigma^{*2}) \\
&= \cO\left(d \log \frac{\norm{A}_2}{\sigma^{*2}} \right) \\
&= \cO\left(d \log \kappa^{\OL}(A) \right)
\end{align*}
as desired. For the inequalities above, the first one follows from the claim~\eqref{eqn:aux_claim}, the second one the fact that $A^{(i)\top} A^{(i)} \preceq A^{(i)\top} A^{(i)} + a_{i+1} a_{i+1}^{\top} = A^{(i+1)\top} A^{(i+1)}$. In the rest of the proof, we prove the claim~\eqref{eqn:aux_claim}.

Suppose that the singular value decomposition of $A^{(i)}$ and $A^{(i+1)}$ are $U^{(i)}\Sigma^{(i)} V^{(i)\top}$ and $U^{(i+1)} \Sigma^{(i+1)} V^{(i+1)\top}$ respectively. Let the singular values of $A^{(i)}$ and $A^{(i+1)}$ be $\sigma_1^{(i)},\ldots, \sigma_d^{(i)}$ and $\sigma_1^{(i+1)},\ldots, \sigma_d^{(i+1)}$ both with descending order. Let $e_{i+1}^{\top}$
be the $(i+1)$-st canonical basis vector and $\tau$ be the leverage score of row $a_{i+1}$ in $A^{(i+1)}$. By the definition of leverage scores, we have $\tau= \norm{e_{i+1} U^{(i+1)}}_2$ . According to Proposition~\ref{prop:leverage-score}, there exists $(\mu,w)\in\R^i\times [0,1]$ such that $(\mu^{\top}\ w)
\begin{pmatrix}
A^{(i)}\\
a_{i+1}^{\top}
\end{pmatrix} = a_{i+1}^{\top}$ and $\norm{\mu}_2^2 + w^2 =\tau$. 
Hence, 
\begin{align*}
\tau_{i+1}^{\OL}(A, \lambda) = a_{i+1}^{\top}(A^{(i)\top}A^{(i)} + \lambda I_d)^{-1} a_{i+1} &= \frac{1}{(1-w)^2} \mu^{\top} A^{(i)} (A^{(i)\top} A^{(i)} + \lambda I_d)^{-1} A^{(i)\top} \mu \\
&= \frac{1}{(1-w)^2} \mu^{\top} U^{(i)} 
\begin{pmatrix}
\frac{\sigma_{1}^{(i)2}}{\sigma_{1}^{(i)2} + \lambda} & & \\
& \ddots &\\
&  & \frac{\sigma_{d}^{(i)2}}{\sigma_{d}^{(i)2} + \lambda}
\end{pmatrix} 
U^{(i)\top} \mu \\
&\leq  \frac{1}{(1-\sqrt{\tau})^2} \mu^{\top} U^{(i)} U^{(i)\top} \mu\\
&= \frac{\tau}{(1-\sqrt{\tau})^2},
\end{align*}
and similarly,
\begin{align*}
a_{i+1}^{\top}(A^{(i+1)\top}A^{(i+1)} + (\sigma^\ast)^2 I_d)^{-1} a_{i+1}                   &= e_{i+1}^{\top} U^{(i+1)} \begin{pmatrix}
\frac{\sigma_{1}^{(i+1)2}}{\sigma_{1}^{(i+1)2} + \sigma^{*2}} & & \\
& \ddots &\\
&  & \frac{\sigma_{d}^{(i+1)2}}{\sigma_{d}^{(i+1)2} + \sigma^{*2}}
\end{pmatrix}
U^{(i+1)\top} e_{i+1}\\
& \geq \frac{1}{2} \tau.
\end{align*}

For notational convenience, let $R$ denote the right-hand side of~\eqref{eqn:aux_claim}. When $\tau \geq \frac{1}{4}$, we have $\tau_{i+1}^{\OL}(A; \lambda)\leq 1 \leq 8R$, where $\tau_{i+1}^{\OL}(A; \lambda) \leq 1$ follows from the Definition~\ref{def:ridge}. When $0 \leq \tau \leq \frac{1}{4}$, it holds that $\tau_{i+1}^{\OL}(A; \lambda) \leq 4\tau \leq 8R$. Hence, it always holds that $\tau_{i+1}^{\OL}(A, \lambda)\leq 8R$ when $0 \leq \lambda \leq (\sigma^{\ast})^2$, establishing the claim.
\end{proof}

At last, we present our proof of Lemma~\ref{lem:sum_of_online_lewis_weights}, following the first part in the proof of Lemma 5.15 in \cite{sliding_window} but with a different argument in the second part.
\begin{proof}\textbf{of Lemma~\ref{lem:sum_of_online_lewis_weights} }
We follow the proof idea of Lemma 5.15 in \cite{sliding_window}. Suppose that $\lambda > 0$. Let $B_0=\lambda I_d$, $B=\underbrace{B_0\circ\dots\circ B_0}_{n \textrm{ times}}$ and $X\triangleq B\circ A$. Let $T$ be the upper bound of the sum of online leverage scores of $A$ with regularization parameter $\lambda$. Following the proof of Lemma 5.15 of \cite{sliding_window}, we have $\sum_{i=1}^n w_i^{\textrm{OL}}(X) = \cO(T\log n) = \mathcal{O}(d\log n \log \kappa^{\OL}(A))$ by Lemma~\ref{lem:sum-online-ridge-leverage-score}.

Now, let $W_A$ be the Lewis weight matrix of $A$ and $L = A^\top W_A^{1-2/p} A$. Let $\sigma = \lambda_{\min}(L)$, the smallest eigenvalue of $L$, and $\rho = \min_i (L^{-1})_{ii}$, the smallest diagonal element of $L_i^{-1}$. Choose $\lambda \leq \left(\frac{\sigma}{n}\right)^{1/p}\rho^{(2-p)/(2p)}$,  $\mu=(\frac{n\lambda^2}{\sigma})^{p/(2-p)}$, $U_X=\mu I_{nd}$ and $W_X=\begin{bmatrix}U_X & \\ & W_A\end{bmatrix}$. We claim that 
\begin{align}
\frac{1}{2} \mu^{2/p}\leq B_j^\top\left(A^\top W_A^{1-2/p}A+B^\top U_X^{1-2/p}B\right)^{-1}B_j, 
\label{eqn:online_lewis_weight_aux1} \\
\frac{1}{2} (w_i(A))^{2/p}\leq a_i^\top\left(A^\top W_A^{1-2/p}A+B^\top U_X^{1-2/p}B\right)^{-1}a_i  \label{eqn:online_lewis_weight_aux2}
\end{align}
for all $j\in[nd]$ and all $i\in[n]$. Observe that $B^\top U_X^{1-2/p}B = n\lambda^2\mu^{1-2/p}I_d \leq \sigma I_d \preceq L$. Thus, 
\[
a_i^\top(L+n\lambda^2\mu^{1-2/p}I_d)^{-1}a_i\geq \frac{1}{2}a_i^\top L^{-1}a_i=\frac{1}{2} (w_i(A))^{2/p},
\]
establishing~\eqref{eqn:online_lewis_weight_aux2}. Similarly, since $B_j = \lambda e_i$ for some $i$,
\[
B_j^\top(L+n\lambda^2\mu^{1-2/p}I_d)^{-1}B_j\geq \frac{1}{2}\lambda^2(L^{-1})_{i,i} \geq \frac{1}{2}\lambda^2 \rho \geq \frac{1}{2}\mu^{2/p},
\]
establishing~\eqref{eqn:online_lewis_weight_aux1}.
It then follows from Lemma \ref{lem:good_Lewis_initial} that $w_i(A) \leq 2w_{nd+i}(X)$. Applying the argument above to the $n$ submatrices which consist of the first $i$ rows of $A$ for each $i=1,\dots, n$, we see that we can choose $\lambda$ to be sufficiently small such that $w_i^{\textrm{OL}}(A) \leq 2w_{nd+i}^{\textrm{OL}}(X)$ for all $i$. Therefore,  $\sum_i w_i^{\textrm{OL}}(A) = \mathcal{O}(d\log n \log \kappa^{\OL}(A))$.
\end{proof}


\subsection{Proof of Lemma~\ref{lem:sum-sub-online-LW}}\label{sec:proof-sum-sub-online-LW}
First, we note the following facts. For any two matrices $A$ and $B$, $\norm{AB}_2 \leq \norm{A}_2 \norm{B}_2$, and when $A$ has full row rank and $B\neq 0$, $\sigma_{\min}(AB)\geq \sigma_{\min}(A) \sigma_{\min}(B)$, where $\sigma_{\min}(\cdot)$ denotes the smallest nonzero singular value of a matrix.

It is clear that $S$, which is a rescaled sampling matrix, has full row rank. By the definition of the online condition number, 
\begin{align*}
\kappa^{\OL}(SA) = \norm{SA}_2 \max_{i} \frac{1}{\sigma_{\min}(SA^{(i)})} &\leq \norm{S}_2 \norm{A}_2 \max_i 
\frac{1}{\sigma_{\min}(SA^{(i)})_{i}}\\
&\leq \norm{S}_2 \norm{A}_2 \max_i \frac{1}{\sigma_{\min}(S) \sigma_{\min}(A)_{i}} \\
&= \frac{\sigma_{\max}(S)}{\sigma_{\min}(S)} \kappa^{\OL}(A).
\end{align*}

Now, observe that $\sigma_{\max}(S) = \max_i p_i^{-1/p} = (\min_i p_i)^{-1/p}$ and $\sigma_{\min}(S) = \min_i p_i^{-1/p} = (\max_i p_i)^{-1/p}$, where $\min\{\beta w_i^{\OL}(A), 1\} \leq p_i\leq 1$. It is clear that $\sigma_{\min}(S)\geq 1$. For the upper bound of $\sigma_{\max}(S)$, note that a row $i$ with $w_i^{\OL}(A)\leq \delta/n$ will be sampled with probability 
\[
1 - \left(1 - \frac{\delta}{n}\right)^{n} \leq \delta.
\]
Hence, with probability at least $1-\delta$, none of the rows $i$ with $w_i^{\OL}(A)\leq \delta/n$ is sampled and so $\min_i p_i \geq \beta\delta/n$ and $\sigma_{\max}(S)\leq (n/(\delta\beta))^{1/p}$. Therefore, we conclude that with probability at least $1-\delta$,
\[
\kappa^{\OL}(SA) \leq \left(\frac{n}{\beta\delta}\right)^{1/p} \kappa^{\OL}(A).
\]

\section{Modifications for \texorpdfstring{$p=2$}{Lg}} \label{sec:p=2 appendix}
When $p=2$, the sampling matrices in Algorithm~\ref{alg:online-active-regression} do not have independent rows since the online leverage scores are calculated with respect to sampled rows instead of all the rows that have been revealed. As a result, we cannot use a Bernstein bound, which is exactly where the proof of \cite[Theorem 4.1]{MMWY} needs to be modified.
Comparing with \cite[Lemma 3.27]{MMWY}, we analyze the sampling process via a martingale because the rows sampled are not independent. Therefore, we shall use Freedman's inequality instead of Bernstein's inequality. This approach was used by \cite{CMP} for online $\ell_2$-regression.

\begin{lemma}\label{lem:modification}
Let $\tilde{w_i}$ be the approximate online Lewis weights of $A$ obtained by Algorithm~\ref{alg:online-active-regression}. Consider the same setting of $A$, $b$ and $\bar{z}$ as Lemma~\ref{lem:3.6}. Let $S$ be the rescaled sampling matrix with respect to $p_i = \min \{ \beta \tilde{w}_i, 1 \}$ and $\beta = \Omega( \frac{1}{\eps^4} (d\log \frac{1}{\eps} + \log\frac{1}{\delta}))$. With probability at least $1-\delta$, it holds that $\Abs{\norm{SA x-S\bar{z}}_2^2- \norm{Ax-\bar{z}}_2^2}= \cO(\eps)R^2$ for all $x\in\R^d$ with $\norm{Ax}_2 \leq R$.
\end{lemma}

\begin{proof}
We prove the lemma by the matrix Freedman inequality. Let $Y_i=\norm{(SA)_{[i]} x - (S\bar{z})_{[i]}}_2^2-\norm{A_{[i]}x-\bar{z}_{[i]}}_2^2$, where $M_{[i]}$ denotes the first $i$ rows of $M$. Also let $Y_0=0$ and $X_i = Y_i-Y_{i-1}$. We claim that $\abs{X_i}$ is uniformly bounded. First, observe that
\[
    \abs{X_i} = \Abs{\Norm{\frac{(\mathbbm{1}_S)_{i}}{\sqrt{p_i}}(a_ix-\bar{z}_i)}_2^2-\norm{a_ix-\bar{z}_i}_2^2} \leq\frac{1}{p_i}\cdot\norm{a_ix-\bar{z}_i}_2^2.
\]
If $i\in\mathcal{B}$, we have that $\bar{z}_i=0$ and it follows from Cauchy-Schwarz inequality that $\norm{a_ix}_2^2\leq w_i^{\OL}(A)  \norm{Ax}_2^2 \leq w^{\OL}_i(A) R^2$. Otherwise, $\norm{a_ix-\bar{z}_i}_2^2 \leq (\frac{1}{\epsilon}+1)^2 w_i^{\OL}(A) R^2$. Since $p_i=\min(\beta w^{\OL}_i(A),1)$, we have $\abs{X_i}\leq\frac{4}{\beta\epsilon^2} R^2$, proving the claim.

For brevity of notation, we denote $\E(\cdot|Y_i,\dots,Y_1)$ by $\E_{i-1}(\cdot)$. Then
\begin{align*}
    \E_{i-1}X_i^2&=\E\left(\Norm{\frac{(\mathbbm{1}_S)_i}{\sqrt{p_i}}(a_ix-\bar{z}_i)}_2^2-\norm{a_ix-\bar{z}_i}_2^2\right)^2\\
    &=\E\left(\frac{(\mathbbm{1}_S)_i}{p_i}-1\right)^2\norm{a_ix-\bar{z}_i}_2^4\\
    &=\left(\frac{1}{p_i}-1\right)\norm{a_ix-\bar{z}_i}_2^4\\
    &\leq\frac{w^{\OL}_i(A)}{p_i} \left(\frac{1}{\epsilon}+1\right)^2 R^2\norm{a_ix-\bar{z}_i}_2^2\\
    &\leq\frac{4}{\beta\epsilon^2} R^2 \norm{a_ix-\bar{z}_i}_2^2.
\end{align*}

Therefore, $\sum_{i=1}^{n}\E_{i-1}X_i^2\leq\frac{4}{\beta\epsilon^2}R^2 \sum_{i=1}^{n}\norm{a_ix-\bar{z}_i}_2^2$.
Since $\norm{Ax}_2^2 \leq R^2$ and $\norm{\bar z}_2^2=\cO(R^2)$, we obtain that  $\sum_{i=1}^{n}\E_{i-1}X_i^2\leq\frac{1}{\beta\epsilon^2}\cO(R^4)$.

It follows from the matrix Freedman inequality \cite{Tropp} and $\beta \geq \frac{2}{\eps^4}(d\log\frac{3}{\eps}+\log\frac{1}{\delta})$ that
\[
  \Pr(\Abs{Y_n}\geq C\eps R^2) \leq\exp\left(\frac{-C^2 \eps^2 R^4}{\frac{1}{\beta\eps^2}\cO(R^4)+\frac{\cO(R^4)}{3\beta\eps}}\right) \leq\exp\left(\frac{-\beta\epsilon^4}{2}\right) \leq \left(\frac{\eps}{3}\right)^d\delta
\]
for a constant $C$ large enough. This implies that for a fixed $x\in\R^d$ such that $\norm{Ax}_2 \le R$,
\[
\Abs{\norm{SA x - S\bar{z}}_2^2 - \norm{Ax - \bar{z}}_2^2} \leq \cO(\eps) R^2
\]
with probability at least $1-(\frac{\eps}{3})^d \delta$.

Next, we make a net argument. Assume, without loss of generality, that $R=1$. Consider the $\epsilon$-net $\mathcal{N}$ of the ellipsoid $\mathbf{B} = \{x\in\R^d: \norm{Ax}_2 \leq 1\}$ endowed with distance $d(x,y) = \norm{A(x-y)}_2$. We can choose $\mathcal{N}$ with at most $(\frac{3}{\epsilon})^d$ points. After applying a union bound over the net, we have that $\Abs{\norm{SA x - S \bar{z}}_2^2 - \norm{Ax - \bar{z}}_2^2} = \cO(\eps)$ holds for all $x\in \mathcal{N}$ simultaneously with probability at least $1 - \delta$. 

For any $x\in \mathbf{B}$, there exists $y\in\mathcal{N}$ such that $\norm{Ax-Ay}_2\leq\epsilon$. According to Lemma~\ref{lem:spc-approx}, when $\beta = \Omega(\log \frac{d}{\delta})$, we have that $\frac{1}{2}\norm{Ax}_2^2 \leq \norm{SAx}_2^2 \leq \frac{3}{2}\norm{Ax}_2^2$ for all $x$ simultaneously with probability at least $1-\delta$, so it holds that $\norm{SA x - SA y}_2 \leq \sqrt{\frac{3}{2}}\eps$. Hence, by the triangle inequality we get
\begin{align*}
    &\quad\ \Abs{\Norm{SA x- S \bar{z}}_2^2 - \Norm{Ax - \bar{z}}_2^2} \\
    &=\Abs{\Norm{SAx - SAy + SAy - S\bar{z}}_2^2 + \Norm{Ax-\bar{z}}_2^2} \\
    & \leq \Abs{\Norm{SAx - SAy}_2^2 + \Norm{SAy - S\bar{z}}_2^2 + 2\langle SAx - SAy, SAy - S\bar{z} \rangle - \Norm{Ax-\bar{z}}_2^2}\\
    & \leq \Abs{\Norm{SAy - S\bar{z}}_2^2 - \Norm{Ax - Ay + Ay - \bar{z}}_2^2} + \Norm{SAx - SAy}_2^2 + 2\Norm{SAx - SAy}_2 \Norm{SAy - S\bar{z}}_2 \\
    & \leq \Abs{\Norm{SA y - S \bar{z}}_2^2 - \Norm{A y - \bar{z}}_2^2} + \Norm{SA x - SA y}_2^2 + \Norm{A x - A y}_2^2 \\
    & \qquad \qquad + 2 \Norm{SAx- SAy}_2 \Norm{SAy-S\bar{z}}_2 + 2\Norm{Ax- Ay}_2 \Norm{Ay-\bar{z}}_2\\
    &\leq \cO(\epsilon) + \frac{3}{2}\eps^2 + \eps^2 + 2\sqrt{\frac{3}{2}}\eps (\Norm{Ay - \bar{z}}_2 + \sqrt{\eps}) + 2\eps \cdot \Norm{Ay - \bar{z}}_2\\
    &\leq \cO(\eps) + 2\eps \cdot (2\sqrt{6} + \sqrt{\eps} + 2) \\
    &= \cO(\eps)
\end{align*}
for all $x$ such that $\norm{Ax}_2\leq 1$, where we have used that $\norm{SAy - S\bar{z}}_2^2 \leq \norm{Ay - \bar{z}}_2^2 + \eps \leq (\norm{Ay - \bar{z}}_2 + \sqrt\eps)^2$ and $\norm{Ay - \bar{z}}_2 \leq \norm{Ay}_2 + \norm{\bar{z}}_2 \leq 2$.
\end{proof}

\bibliographystyle{plain}
\bibliography{reference}

\end{document}